\documentclass{article}
\usepackage[utf8]{inputenc}

\usepackage{mathrsfs}  
\usepackage{microtype}
\usepackage{fullpage}
\usepackage{graphicx}
\usepackage{booktabs} 
\usepackage{changes}

\usepackage{hyperref}

\newcount\Comments  
\newcommand{\kibitz}[2]{\ifnum\Comments=1{\color{#1}{#2}}\fi}

\usepackage{color}
\definecolor{english}{rgb}{0.0, 0.5, 0.0}

\usepackage{amsmath}
\usepackage{amssymb}
\usepackage{mathtools}
\usepackage{amsthm}

\usepackage[capitalize,noabbrev]{cleveref}

\usepackage[switch]{lineno} 
\usepackage[]{algorithm}

\usepackage[]{algorithmic}

\usepackage{graphicx}
\usepackage{subcaption}
\usepackage{amsmath}
\usepackage{amssymb}
\usepackage{mathtools}
\usepackage{amsthm}
\usepackage{amsmath,amsfonts,amssymb,amsthm}
\usepackage{amsmath,amssymb,amsthm}
\usepackage{amsmath,amssymb,mathtools}
\usepackage{amsmath,amssymb,amsthm}
\usepackage{amsmath,amssymb,mathtools}
\usepackage{bm}
\usepackage{amsmath}
\usepackage{amssymb}
\usepackage[capitalize,noabbrev]{cleveref}

\usepackage{xcolor}
\usepackage[normalem]{ulem}
\usepackage{mwe}

\usepackage{psfrag}

\usepackage{mathrsfs}

\usepackage{tikz}
\usetikzlibrary{calc}

\renewcommand{\cal}{\mathcal}

\newcommand{\cD}{{\cal D}}

\newcommand{\cN}{{\cal N}}

\newcommand{\cP}{{\cal P}}

\newcommand{\sD}{{\mathscr{D}}}

\newcommand{\bE}{\mathbb{E}}

\newcommand{\bP}{\mathbb{P}}

\renewcommand{\leq}{\leqslant}
\renewcommand{\geq}{\geqslant}

\theoremstyle{plain}
\newtheorem{theorem}{Theorem}[section]

\newtheorem{lemma}[theorem]{Lemma}

\theoremstyle{definition}
\newtheorem{definition}[theorem]{Definition}

\theoremstyle{remark}

\DeclareMathOperator{\argmin}{argmin}

\title{Decision-Aware Conditional GANs for Time Series Data}
\author{He Sun\textsuperscript{\rm 1}
        ~~Zhun Deng\textsuperscript{\rm 1}
        ~~Hui Chen\textsuperscript{\rm 2}
        ~~David C. Parkes\textsuperscript{\rm 1}\\}
\date{\textsuperscript{\rm 1} Harvard University 
    \textsuperscript{\rm 2} Massachusetts Institute of Technology \\
    he\_sun@g.harvard.edu, zhundeng@g.harvard.edu, huichen@mit.edu, parkes@eecs.harvard.edu}

\begin{document}

\maketitle

\begin{abstract}
We introduce the  {\em decision-aware time-series conditional generative adversarial network} (DAT-CGAN), a method for time-series generation. The framework adopts a multi-Wasserstein loss on decision-related quantities and is designed to support decision-making. We use an
 overlapped block-sampling approach for sample efficiency, and
 characterize  the generalization properties of DAT-CGAN. In application 
to financial time series and a multi-period portfolio choice problem, we
 demonstrate better training stability and generative quality in regard to both raw data and decision-related quantities than strong GAN-based baselines.
\end{abstract}

\section{Introduction}
\label{sec:introduction}

High-fidelity simulators are  desirable across many domains,  due to a paucity of data, or because of high stakes in the deployment of automated decision methods. A good simulator can improve sample efficiency
 and performance when evaluating a decision method.  
A gap in current GAN-based approaches, however, is that they are not {\em decision-aware}, but focus instead on modeling
 the raw synthetic data distribution~\cite{Koshiyama2019,Yoon2019}. The
 issue is that the approximation error in the raw synthetic data distribution 
can  be further amplified, after the transformation through decision-related functions,  making the approximation to the underlying distribution of decision-related quantities unreliable.

We show that taking decision-related quantities into account during training
improves the effectiveness of GANs in support of decision making and
 helps to stabilize the  training process. We do this by introducing   {\em decision-related quantities} to the loss function. We focus on time-series data, and consider settings where decisions are made over time 
and based on multi-step inference.
In such settings, data scarcity can be a particular challenge, either due to a limited sample size or non-stationarity.\footnote{

For example, weekly financial data series only provide around 50 observations per year, and pooling across multiple years is not effective due to potential distributions shifts. Even for high-frequency data, distribution shift remains a concern due to the potential changes in the composition of market participants and trading rules over time.}
Previous work has applied bootstrap methods to improve the sample efficiency of estimators for time-series data~\cite{Buhlmann2015}, but without introducing decision-related quantities or providing finite-sample guarantees on the generalization error.

Another challenge when training sequential generative models is  {\em exposure bias}~\cite{Ranzato15}.  First studied in language models, this  arises when models are trained to predict one-step forward using previous ground truth observations, whereas at test time they are used to generate an entire sequence. As a result, the generated data distribution can diverge from the training distribution, with accumulating errors. Although exposure bias has received attention in language models~\cite{Bengio2015,Rennie2017}, this problem remains present in other generative model applications~\cite{Li2019}. 

\noindent\textbf{Our Contributions.} We propose a novel, {\em decision-aware time-series conditional generative adversarial network (DAT-CGAN)}. The  training procedure is made decision aware by imposing a {\em multi-Wasserstein loss structure}, which includes loss terms on  decision-related quantities in addition to the raw data. We remove exposure bias by aligning training and evaluation with the same number of look-ahead steps. In improving sample efficiency, we adopt an {\em overlapped block-sampling mechanism}. The design of the generator and discriminator in the DAT-CGAN is non-trivial since the generator needs to capture the structural relationship between different decision-related quantities. We provide the discriminator with access to the same amount of conditioning information as the generator to avoid it being too strong relative to the generator. 

We provide a theoretical characterization,
giving non-asymptotic generalization
bounds for Conditional GANs for time series with overlapping block sampling, and with a loss function that includes decision-related quantities.
In experimental results, we evaluate our framework on a portfolio choice problem for a risk-averse investor~\cite{Markowitz1952}. A high-quality simulator will help an investor characterize the distribution of portfolio returns more reliably than relying on simplistic parametric assumptions such as normality, for example
in regard to computing   measures such as {\em Value-at-Risk}
that  are important for risk management and financial regulation. We demonstrate on both simulated and real data 
that  introducing portfolio-decision-related quantities into the loss function, 
in addition to asset returns, 
the DAT-CGAN framework achieves better fit to quantities of interest than strong, GAN-based baselines.

\noindent\textbf{Related Work.} The  literature on  GANs for time-series data does not consider decision awareness,  even for the non-time series context, and does not provide theoretical guarantees for  decision-related quantities or  conditional GANs with overlapped-sampling schemes~\cite{Takahashi2019,Zhou2018,Wiese2019,Ni2020,Yoon2019}. In the context of financial markets, \cite{Li2019} introduce  {\em stock-GAN}  for the  generation of order streams, and evaluate their approach on stylized facts about  market micro-structure. \cite{Koshiyama2019} study the use of GANs for generating additional synthetic samples to help 
with model calibration and aggregation. Neither method is decision aware, and they are  similar to the baselines we use in
 ablations (specifically, they are akin to a ``1-step, return-data only" baseline). There is also work that makes use of GANs  to perform anomaly detection in time series data~\cite{Li2019b,Li2018,Geiger2020,Bashar2020},  
for imputation for multivariate time series~\cite{Luo2018,Luo2019},
and  to study  causal  effects in economic models~\cite{Athey2019}. 

\section{Wasserstein GANs for Time Series Generation}
A Wasserstein GAN uses the {\em  Wasserstein distance} as the loss function. The Wasserstein distance between two random variables, $r$ and $r'$, distributed according to  $\mathcal{P}_r$ and $\mathcal{P}_{r'}$,
is 
\begin{equation*}
    W(\mathcal{P}_r, \mathcal{P}_{r'}) = \mbox{inf}_{\Gamma \in \Pi(\mathcal{P}_r, \mathcal{P}_{r'})}\mathbb{E}_{(r,r') \sim \Gamma}[\|r-r'\|],
\end{equation*}
where $\|\cdot\|$ is the $L^2$ norm, and $\Pi(\mathcal{P}_r, \mathcal{P}_{r'})$ is the set of all joint distributions  whose marginals equal to $\mathcal{P}_r$ and $\mathcal{P}_{r'}$. According to the  Kantorovich-Rubinstein duality~\cite{Villani2009}, the dual form can be written as:
\begin{equation*}
    \mbox{sup}_{\{h:\|h\|_L \leq 1\}} \mathbb{E}_{r\sim \mathcal{P}_r}[h(r)] - \mathbb{E}_{r'\sim \mathcal{P}_{r'}}[h(r')],
\end{equation*}
where $\|h\|_L$ is defined as $\sup_{x,x'}|h(x)-h(x')|/\|x-x'\|$.
For Wasserstein GANs, the goal is to minimize the Wasserstein distance between the non-synthetic data and synthetic data.\footnote{We use Wasserstein loss  because it tends to
 improve the learning process stability  relative to other choices,  for example in regard to  mode collapse,
 and to yield interpretable learning curves~\cite{Arjovsky2017}.} Following~\cite{Mirza2014}, we work with {\em Conditional GANs}, allowing for conditioning variables. For  functions $D_{\theta}$ and $G_{\eta}$, the {\em discriminator}
parameterized by $\theta$ and
the {\em generator}  parametrized by $\eta$, and
with  conditioning variable $x$, the  CGAN problem is
\begin{equation*}
    \mbox{min}_{\eta}\mbox{max}_{\theta} \mathbb{E}_{r\sim \mathcal{P}(r|x)}[D_{\theta}(r, x)] - \mathbb{E}_{z\sim \mathcal{P}(z)}[D_{\theta}(G_{\eta}(z, x), x)],
\end{equation*}
where $\mathcal{P}(r|x)$ and  $\mathcal{P}(z)$  denote the distribution of non-synthetic data and  input random seed, respectively. Here, the synthetic data comes from the generator, with  $r' = G_{\eta}(z, x)$ 
conditioning on $x$.

\section{Decision-Aware Time Series Conditional Generative Adversarial Network}


Let $(r_1,\ldots, r_T)$ denote a multivariate time series, where $r_t$ in time $t$ is a $d$-dimensional column vector. Let $x_t$ denote an $m$-dimensional  {\em time-series information vector}, summarizing relevant information up to time $t$. Let $R_{t+1:t+k} = (r_{t+1}, \ldots, r_{t+k})$ denote  a {\em $k$-length  block after time $t$}, where $k\in \{1,\ldots,K\}$ is the $k$th look ahead step, and $K$ is the total number of {\em look-ahead steps} to generate. Since in  applications such as finance the sample size is very limited,  we let the $R_{t+1:t+k}$ blocks   overlap with each other for different $t$ and $k$, in order to fully utilize the  samples.  In finance, $r_{t+1}$ could be the  asset returns at day $t+1$, $x_t$  the past asset returns, volatility, and other technical indicators, and
 $R_{t+1:t+k}$  the $k$-days forward asset returns (see Figure~\ref{fig_overlapped_sampling3}).

To model the  decision process of an end user, let $f_{j,k}(R_{t+1:t+k}, x_t)$ denote the $j$th 
decision-related quantity (a scalar, vector, or matrix), for $j\in \{1,\ldots,J\}$,
in look-ahead step $k$, and define $f_k(R_{t+1:t+k}, x_t)$ as $(f_{1,k}(R_{t+1:t+k}, x_t), \ldots, f_{J,k}(R_{t+1:t+k}, x_t))^{\intercal}.$ The above definition represents the $J$ decision-related quantities at look-ahead step $k$ given data $R_{t+1:t+k}$ and  information $x_t$. In finance, $f_{j,k}(R_{t+1:t+k}, x_t)$ could be the estimated co-variance of asset returns, 
or the portfolio weights, both determined using the information up to time $t+k$. 

\smallskip


\noindent\textbf{Multi-Wasserstein loss.} Let $r'_{t+k|t}$ denote the {\em synthetic data} generated from information vector $x_t$  for look-ahead step $k$, for $k\in \{1,\ldots,K\}$. We use  notation $r'_{t+k|t}$ rather than notation $r'_{t+k}$ because there is a difference, for example, between $r'_{12|9}$ and $r'_{12|10}$, where $r'_{12|9}$ is the synthetic data generated for day 12 conditioning on  information up to day 9, and $r'_{12|10}$  is  the synthetic data  for day 12 conditioning on  information up to day 10. For all  $t$, all $k$, we want the conditional distribution  on  synthetic data, $\mathcal{P}(r'_{t+k|t}|x_t)$, where $x_t$ is discrete, to match the conditional distribution on the non-synthetic data, $\mathcal{P}(r_{t+k}|x_t)$. Similarly, for all $t$, all $k$, we want the conditional distribution on decision-related quantities for synthetic data, $\mathcal{P}(f_{j,k}(R'_{t+1:t+k}, x_t)|x_t)$, where  $R'_{t+1:t+k} = (r'_{t+1|t}, \ldots, r'_{t+k|t})$, to match the conditional distribution, $\mathcal{P}(f_{j,k}(R_{t+1:t+k}, x_t)|x_t)$, on quantities computed for non-synthetic data; see Figure~\ref{fig_overlapped_sampling3}. It will be convenient to write $\mathcal{P}(R'_{t+1:t+K}|x_t)$ for $\{\mathcal{P}(r'_{t+k|t}|x_t)\}_{k\in \{1,\ldots,K\}}$.

 Adopting a separate loss term for each  quantity and  each look-ahead step $k$, we define the following multi-Wasserstein objective (written here for conditioning, $x_t$): 
\begin{align}
&\mbox{inf}_{\mathcal{P}(R'_{t+1:t+K}|x_t)} \sum_{k=1}^K \omega_k L^r_{k} + \sum_{k=1}^K\sum_{j=1}^J \lambda_{j,k} L^f_{j,k}, \label{equ:1}\\
&L^r_{k}= W\big(\mathcal{P}(r_{t+k}|x_t), \mathcal{P}(r'_{t+k|t}|x_t)\big)\notag \\
&L^f_{j,k} = W\big(\mathcal{P}(f_{j,k}(R_{t+1:t+k}, x_t)|x_t),\notag \\
&~~~~~~~~~~~\mathcal{P}(f_{j,k}(R'_{t+1:t+k},x_{t})|x_t)\big)\notag,
\end{align}
where $L^r_{k}$ denotes the loss for data at $k$ steps forward,  $L^f_{j,k}$  the loss for decision-related quantity $j$ at $k$ steps forward, and where values $\omega_k>0$ and $\lambda_{j,k}>0$ are weights.\footnote{An alternative formulation would impose the Wasserstein distance  on a vector concatenating all quantities. 
We justify this design choice  in the experimental results section.}

\smallskip

\noindent\textbf{Surrogate loss.}
Let $D_{\gamma_k}$  denote the discriminator at look-ahead step $k$,  with 
parameters
$\gamma_k$, and  $D_{\theta_{j,k}}$ 
   the discriminator for decision-related quantity $j$ at look-ahead step $k$,  with 
parameters $\theta_{j,k}$.
Let  $r'_{t+k|t} = G_{\eta}(z_{t,t+k}, x_t)$ denote the synthetic  data at look-ahead step $k$, where 
$G_{\eta}$ is the generator with parameters $\eta$, 
and where noise $z_{t,t+k} \sim N(0,I_d)$. Let $Z_{t,t+k} = (z_{t,t+1}, \ldots, z_{t,t+k})$ denote a $k$-length block of random seeds after $t$. We define the following in-expectation quantities:
\begin{align}
&\mathbb{E}^r_{k} = \mathbb{E}_{r_{t+k} \sim \mathcal{P}(r_{t+k}|x_t)}[D_{\gamma_k}(r_{t+k}, x_t)], 
\label{eq:Er} \\ 
 &\mathbb{E}^{G_{\eta}}_{k} = \mathbb{E}_{z_{t,t+k}\sim \mathcal{P}(z_{t,t+k})}[D_{\gamma_k}(r'_{t+k|t}, x_t)],   \label{eq:Eg}\\
&\mathbb{E}^{f,R}_{j,k} = \mathbb{E}_{R_{t+1:t+k} \sim \mathcal{P}(R_{t+1:t+k}|x_t)} \nonumber\\
&~~~~~~~~~~~~~~~~[D_{\theta_{j,k}}(f_{j,k}(R_{t+1:t+k}, x_t), x_t)], \label{eq:Efr} \\  
&\mathbb{E}^{f,G_{\eta}}_{j,k} = \mathbb{E}_{Z_{t,t+k} \sim \mathcal{P}(Z_{t,t+k})} \nonumber\\
&~~~~~~~~~~~~~~~~[D_{\theta_{j,k}}(f_{j,k}(R'_{t+1:t+k}, x_t), x_t)],  \label{eq:Efg}  
\end{align}
where~\eqref{eq:Er} and~\eqref{eq:Efr} are defined on non-synthetic data and~\eqref{eq:Eg} and~\eqref{eq:Efg} on synthetic data, and~\eqref{eq:Efr} and~\eqref{eq:Efg} are defined on derived, decision-related quantities.

To formulate the DAT-CGAN training problem, we use the Kantorovich-Rubinstein duality for each Wasserstein distance in~\eqref{equ:1}, and  sum over  the dual forms \cite{Villani2009}.
%
This provides a {\em  surrogate loss}, upper bounding  the original objective. 
The surrogate problem is a min-max optimization problem, with
 the {\em discriminator loss} defined as:
\begin{equation*}
\inf_{\eta} \sup_{\gamma_k, \theta_{j,k}} \sum_{k=1}^K \omega_k (\mathbb{E}^r_{k} - \mathbb{E}^{G_{\eta}}_{k}) + 
\sum_{k=1}^K\sum_{j=1}^J \lambda_{j,k} (\mathbb{E}^{f,R}_{j,k} - \mathbb{E}^{f,G_{\eta}}_{j,k}).
\end{equation*}
The {\em generator loss} is $$\inf_{\eta} - \sum_{k} \omega_k \mathbb{E}^{G_\eta}_{k} - \sum_{k,j} \lambda_{j,k} \mathbb{E}^{f,G_{\eta}}_{j,k}.$$
We also write  $\tilde{L}^r_{k} = \mathbb{E}^r_{k} - \mathbb{E}^G_{k}$  and   $\tilde{L}^f_{j,k} = \mathbb{E}^{f,R}_{j,k}-\mathbb{E}^{f,G_{\eta}}_{j,k}$, 
to denote the {\em discriminator loss} for  the raw data and decision-related quantities, respectively.

\noindent\textbf{Training procedure.} See Algorithm~\ref{alg1}.
Lines~2-3 
 prepare the data.
Lines~6-15 
train the discriminators: Line~7 
performs $K$-length block sampling;  Lines~8-10 
generate synthetic block samples for each time block, conditioning on the information vector; 
Lines~11-14 
update the discriminators. Lines~16-21 
train the generators: Lines~17-19 
generate synthetic block samples for each time block, conditioning on the information vector;
 Line~20 
updates the generators.

 We define sample estimates for expectations~\eqref{eq:Er},~\eqref{eq:Eg},~\eqref{eq:Efr},~\eqref{eq:Efg},   as
$\hat{\mathbb{E}}^r_k$, $\hat{\mathbb{E}}^{G_{\eta}}_k$, $\hat{\mathbb{E}}^{f,R}_{j,k}$ and $\hat{\mathbb{E}}^{f,G_{\eta}}_{j,k} $, respectively. Quantities $(r_{t_i+k}, f_k(R_{t_i+1:t_i+k}, x_{t_i}), x_{t_i}), \forall i$ are obtained by an overlapped block sampling scheme (see Figure \ref{fig_overlapped_sampling3}), where different blocks of samples can overlap with other blocks. 

\begin{figure}[t]
\centering
\includegraphics[width = .7\textwidth]{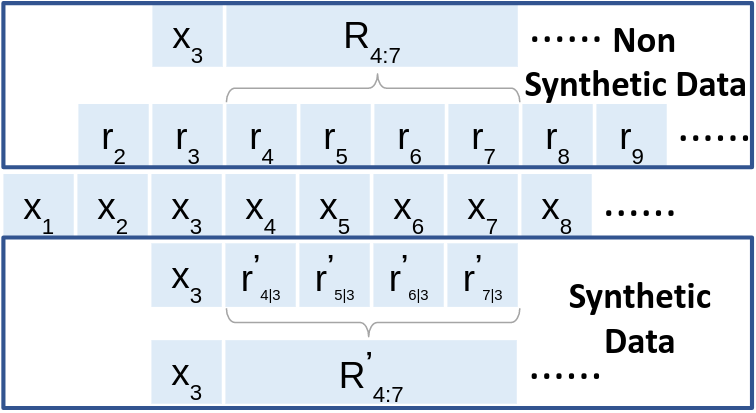}
\caption{Non-synthetic data and synthetic data generated by conditional GANs, shown for $K=4$. 
\label{fig_overlapped_sampling3}}
\end{figure}

\begin{algorithm}[t]
   \caption{\label{alg1}. Learning Rate $\alpha = 1e-5$, $\omega_k = \lambda_{j,k} = 0.8^k$, $s_D = 1$, $s_G = 5$, clipping $l_b = -0.5, u_b = 0.5$, $T=3500$, look ahead step $K=4$, batch size $I = 32$, training steps $N = 2e5$.}
\begin{algorithmic}[1]
   \STATE {\bfseries Require:} $\gamma_{k,0}$ and $\theta_{j,k,0}$, initial discriminator parameters; $\eta_0$, initial generator parameters. 
   \FOR {$t = 1, k = 1$ {\bfseries to} $T, K$}\label{alg_prepare_begin}
        \STATE Compute $R_{t+1:t+k}$ and $f_{j,k}(R_{t+1:t+k}, x_t), \forall j$ 
   \ENDFOR\label{alg_prepare_end}
   \WHILE {$n < N$}
   \FOR {$s = 0$ {\bfseries to} $s_D$}\label{alg_train_dis_begin}
   \STATE Make $I$  samples of $K$-size time blocks. \\The $i$th sample ($1 \leq i \leq I$)  ranges from time $t_i + 1$ to $t_i + K$, and consists of  data
   $(r_{t_i+k}, f_k(R_{t_i+1:t_i+k}, x_{t_i}), x_{t_i})^K_{k=1}$ \label{alg_block_sampling}
      \FOR {$i = 1, k = 1$ {\bfseries to} $I, K$}\label{alg_gen_syn_block_in_dis_train_begin}
        \STATE Sample $z_{t_i,t_i+k} \sim \mathcal{P}(z_{t_i,t_i+k})$; Compute $r'_{t_i+k|t_i} = G_{\eta}(z_{t_i,t_i+k}, x_{t_i})$; Compute $f_{j,k}(R'_{t_i+1:t_i+k}, x_{t_i}), \forall j$
      \ENDFOR\label{alg_gen_syn_block_in_dis_train_end}   
      \FOR {$k = 1$ {\bfseries to} $K$}\label{alg_update_dis_begin}
     \STATE  $\gamma_{k} \leftarrow \mbox{clip}(\gamma_{k} + \alpha \omega_k \nabla_{\gamma_k} \sum_{k=1}^{K} [\hat{\mathbb{E}}^r_k - \hat{\mathbb{E}}^{G_{\eta}}_k], l_b, u_b)$
      \STATE $\theta_{j,k} \leftarrow \mbox{clip}(\theta_{j,k} + \alpha \lambda_{j,k} \nabla_{\theta_{j,k}} \sum_{k=1}^{K} [\hat{\mathbb{E}}^{f,R}_{j,k} - \hat{\mathbb{E}}^{f,G_{\eta}}_{j,k}], l_b, u_b), \forall j$
      \ENDFOR\label{alg_update_dis_end}

   \ENDFOR\label{alg_train_dis_end}
   \FOR {$s = 0$ {\bfseries to} $s_G$}\label{alg_train_gen_begin}
     \FOR {$i = 1, k = 1$ {\bfseries to} $I, K$}\label{alg_gen_syn_block_in_gen_train_begin}
        \STATE Sample $z_{t_i,t_i+k} \sim \mathcal{P}(z_{t_i,t_i+k})$; Compute $r'_{t_i+k|t_i} = G_{\eta}(z_{t_i,t_i+k}, x_{t_i})$; Compute $f_{j,k}(R'_{t_i+1:t_i+k}, x_{t_i}), \forall j$
      \ENDFOR\label{alg_gen_syn_block_in_gen_train_end}
   \STATE $\eta \leftarrow \eta - \alpha \omega_k \nabla_{\eta} \sum_{k=1}^{K} [\hat{\mathbb{E}}^{G_{\eta}}_k - \hat{\mathbb{E}}^{f,{G_{\eta}}}_{j,k}], \forall j$\label{alg_update_gen}

   \ENDFOR\label{alg_train_gen_end}
   \ENDWHILE

\end{algorithmic}
\end{algorithm}

\section{Theoretical Results}
\label{sec:theory}

In this section, we provide a theoretical characterization of the generalization ability of our algorithm. Previous techniques for generalization bounds for standard GANs training with i.i.d.~data \cite{Arora2017}  have not  considered the overlapping structures or derived quantities in our algorithm. To provide generalization bounds, we greatly extend the previous arguments  to adapt to our case. 
%

\cite{Arora2017} have shown that training results  for GANs that appear successful  may be far from the target distribution in terms of standard metrics, such as Jensen-Shannon (JS) divergence or Wasserstein distance---even though the synthetic data distribution is close to the empirical distribution induced by the samples, it can still be far from the underlying true distribution under those metrics. 
For example, the Wasserstein distance between two empirical distributions $\hat\mu$ and $\hat \nu$, both induced by $m$ training samples, can be $0$, while the distance between the true underlying corresponding distributions $\mu$ and $\nu$ can be larger than a constant unless $m$ is exponentially large w.r.t.~the data dimension, which is usually very high. However, in practice, generalization  occurs with respect to a weaker version of distance, i.e. \textit{neural network distance}, defined in Definition~\ref{def:nnd}. In practice, when training WGAN, we are not optimizing the real Wasserstein distance between synthetic and real data, but  a distance approximated by neural networks:
 \begin{equation}
    \mbox{min}_{\eta}\mbox{max}_{\theta} \mathbb{E}_{r\sim \mathcal{P}_r}[D_{\theta}(r)] - \mathbb{E}_{z\sim \mathcal{P}(z)}[D_{\theta}(G_{\eta}(z))], 
\end{equation}
where $D_\theta$ and $G_\eta$ are instantiated as  neural networks, parameterized by $\theta$ and $\eta$.  Here, $r$ is the real data while $G_\eta(z)$ is the synthetic data with random seed $z$. \cite{Arora2017} 
consider the following, weaker metric:
\begin{definition}[Neural Net Distance for WGAN]\label{def:nnd}
Given a family of neural networks $\{D_\theta:\theta\in\Theta\}$ for a set $\Theta$, for two distributions $\mu$ and $\nu$, the corresponding {\em neural network distance} for  the Wasserstein GAN is defined as,
\begin{align}
{\sD}_\Theta(\mu,\nu)&=\sup_{\theta\in\Theta}\{\mathbb{E}_{x\sim\mu}[D_\theta(x)]-\mathbb{E}_{x\sim\nu}[D_\theta(x)]\}.
\end{align}
\end{definition}

With this,  \cite{Arora2017} build a generalization theory for WGANs under the following generalization property.
\begin{definition}
Let $\mathcal{P}_{\text{data}}$ denote the distribution of non-synthetic data and $\mathcal{P}_{G}$ denotes the generated distribution, and let $\hat{\mathcal{P}}_{\text{data}}$ and $\hat{\mathcal{P}}_{G}$ denote the corresponding empirical versions, the generalization gap for WGAN is defined as 
\begin{align}
|{\sD}_\Theta(\hat{\mathcal{P}}_{\text{data}},\hat{\mathcal{P}}_{G})&-{\sD}_\Theta(\mathcal{P}_{\text{data}},\mathcal{P}_{G})|.
\end{align}
\end{definition}

A natural question in our setting is the following:

\textit{Question: for DAT-CGAN, can we give a generalization property guarantee under the neural network distance?}

To build such a theory for DAT-CGAN, instead of dealing with i.i.d.~data as in \cite{Arora2017}, we need to deal with time series and overlapping block sampling as well as the conditioning information. In this section, we will show how to conquer such issues and provide a theoretical guarantee.
Instead of considering a multi-step multi-loss, it is WLOG to consider the case when $k=K$ and  a single decision-related quantity (the raw data can also be viewed as a decision-related quantity, where the corresponding $f$ is the mapping picking the last element in $R_{t_i+1:t_i+K}$). For multiple but finite  values of $k$, and multiple but finite decision-related quantities, we can  use a uniform bound to obtain the corresponding generalization bounds.
%
Given this, we can simplify notation: let $\hat{\mathcal{P}}_R(I)$ and $\hat{\mathcal{P}}_{G_\eta,Z}(I)$ denote the empirical distribution induced by data set $\{(f(R_{t_i+1:t_i+K}, x_{t_i}), x_{t_i})\}_{i=1}^I$ and $\{(f(R'_{t_i+1:t_i+K}, x_{t_i}), x_{t_i})\}_{i=1}^I$, respectively. Recall 
$$R'_{t_i+1:t_i+K} = ( G_{\eta}(z_{t_i,t_i + 1}, x_{t_i}), \ldots, G_{\eta}(z_{t_i,t_i + K}, x_{t_i})),$$
and define
\begin{align}
&{\sD}_\Theta(\hat{\mathcal{P}}_R(I),\hat{\mathcal{P}}_{G_\eta,Z}(I))= \sup_{ \theta\in\Theta}  [\hat{\mathbb{E}}^{f,R} -  \hat{\mathbb{E}}^{f,G_{\eta}}],\\
&\hat{\mathbb{E}}^{f,R} =(1/I)\sum_{i=1}^I [D_{\theta}(f(R_{t_i+1:t_i+K}, x_{t_i}), x_{t_i})],\notag \\
&\hat{\mathbb{E}}^{f,G_{\eta}} =(1/I)\sum_{i=1}^I[D_{\theta}(f(R'_{t_i+1:t_i+K}, x_{t_i}), x_{t_i})].\notag 
\end{align}

Here, $\Theta$ and $\Xi$ are parameter sets. Before formally stating the theoretical results, we need to understand the convergence point of  ${\sD}_\Theta(\hat{\mathcal{P}}_R(I),\hat{\mathcal{P}}_{G_\eta,Z}(I))$.
Notice that for the surrogate loss, taking expectation with respect to $\mathcal{P}(r_{t+K}|x_t)$, for any realization of $x_t$, i.e. $x_t=c$ for constant vector $c$, we need enough samples for $r_{t+K}$ given $x_t=c$ so that the empirical distribution $\hat{\mathcal{P}}(r_{t+K}|x_t=c)$ can  well represent the ground-truth distribution $\mathcal{P}(r_{t+K}|x_t=c)$. However, in applications, we would not normally have enough samples for any arbitrary value $c$, and especially considering that $x_t$ may be a continuous random vector instead of a categorical one. It is even possible that for all $\{t_i\}_{i=1}^I$, the $\{x_{t_i}\}_{i=1}^I$ values are different from each other. Thus, we need to understand what ${\sD}_\Theta(\hat{\mathcal{P}}_R(I),\hat{\mathcal{P}}_{G_\eta,Z}(I))$ converge to as $I\rightarrow \infty$.
We  show that ${\sD}_\Theta(\hat{\mathcal{P}}_R(I),\hat{\mathcal{P}}_{G_\eta,Z}(I))$  converges to a ``weaker" version for a given $\eta$ under certain conditions, i.e., that it converges to
%
\begin{equation}
{\sD}_\Theta(\mathcal{P}_R,\mathcal{P}_{G_\eta,Z})=\sup_{ \theta\in\Theta} [\mathbb{E}^{f,R} -  \mathbb{E}^{f,G_{\eta}}],
\end{equation}
where $\mathcal{P}_R$ and $\mathcal{P}_{G_\eta,Z}$ are  the distribution of $(f(R_{t+1:t+K}, x_{t}), x_{t})$ and $(f(R'_{t+1:t+K}, x_{t}), x_{t})$, respectively, and 
\begin{align*}
\!\!\!\!\!\!\!\!\mathbb{E}^{f,R} &\!=  \mathbb{E}_{x_t}\!\mathbb{E}_{R_{t+1:t+K} \sim \mathcal{P}(R_{t+1:t+K}|x_t)}\\
&~~~~~~~~\![D_{\theta}(\!f(R_{t+1:t+K}, x_t), \!x_t)],\\
\mathbb{E}^{f,G_{\eta}}& =  \mathbb{E}_{x_t}\mathbb{E}_{Z_{t,t+K}\sim \mathcal{P}(Z_{t,t+K})}\\
&~~~~~~~~[D_{\theta}(f(R'_{t+1:t+K}, x_t), x_t)].
\end{align*}

Compared with the surrogate loss mentioned previouly such as Eq.~(\ref{eq:Er}), there is an extra expectation over $x_t$ in ${\sD}_\Theta(\mathcal{P}_R,\mathcal{P}_{G_\eta,Z})$, which comes from sampling over different $\{x_{t_i}\}$'s. We can view this as an average version of the surrogate losses under different realizations of $x_t$'s.  Now we are ready to state a generalization bound regarding $|{\sD}_\Theta(\hat{\mathcal{P}}_R(I),\hat{\mathcal{P}}_{G_\eta,Z}(I)))-{\sD}_\Theta(\mathcal{P}_R,\mathcal{P}_{G_\eta,Z})|$.
In order to conquer the issues with non-i.i.d. data and overlapping sampling, we  introduce a framework for defining suitable {\em mixing conditions}. This kind of framework is commonly used in time-series analysis~\cite{Bradley2008}.
\smallskip

\noindent\textbf{Mixing condition framework.} Let $X_i\in S$ for some set $S$, and $X=(X_1,,\cdots,X_n)$. We further denote $X^j_i=(X_i,X_{i+1},\cdots,X_j)$ as a random vector for  $1\leq i< j\leq n$.  Correspondingly, we let  $x^j_i=(x_i,x_{i+1},\cdots,x_j)$ be a subsequence for the realization of $X$, i.e. $(x_1,x_2,\cdots,x_n)$. We denote the set $\mathcal{C}=\{\bm{y}\in S^{i-1}, w,w'\in S:\mathbb{P}(X^i_1=\bm{y}w)>0,~\mathbb{P}(X^i_1=\bm{y}w')>0\}$, and write $\bar{\eta}_{i,j}(\{X_i\}_{i=1}^n)=\sup_{\mathcal{C}}\eta_{i,j}(\bm{y},w,w'),$ where $\eta_{i,j}((\{X_i\}_{i=1}^n,\bm{y},w,w')$ denotes
\begin{align}
\mathit{TV}\!\Big(\mathcal{P}(X^n_j|X^i_1&=\bm{y}w),\mathcal{P}(X^n_j|X^i_1=\bm{y}w')\Big).
\end{align}

Here, $\mathit{TV}$ is the {\em  total variation distance}, and $\mathcal{P}(X^n_j|X^i_1=\bm{y}w)$ is the conditional distribution of $X^n_j$, conditioning on $\{X^i_1=\bm{y}w\}$.

\noindent\textbf{Assumptions and implications.} First, we make a number of natural, boundedness assumptions. We assume the time series data are of bounded support, i.e. there exists a universal $B_r$, such that $ \max\{\|r_i\|_\infty,\|r_i\|\}\leq B_{r}$,
%
where the boundedness of $\|r_i\|$ is implied by the boundedness of$\|r_i\|_\infty$ since the dimension of $r_i$ is finite. We also assume boundedness of conditioning information $\{x_t\}_t$, i.e.  there exists a universal $B_x$, such that $\max\{\|x_t\|_\infty,\|x_t\|\}\leq B_{x}.$
For the discriminators $D_\gamma$ and $D_\theta$, where $\theta\in\Theta\subseteq \mathbb{R}^p$,  we assume w.l.o.g.~that $\Theta$ is a subset of unit balls with corresponding dimensions.\footnote{We can always rescale the parameter properly by changing the parameterization as long as $\Theta$ is bounded. The boundedness of $\Theta$ is naturally satisfied since the training algorithm of the Wasserstein GAN requires weight clipping.} Similarly, for the generative model $G_{\eta}$, $\eta\in \Xi$, we  assume $\Xi$ is a subset of unit ball. 

We require $L$-Lipschitzness of $D_\theta$ and $G_{\eta}$ with respect to their parameters, i.e. $\|D_{\theta_1}(x)-D_{\theta_2}(x)\|\leq L \|\theta_1-\theta_2\|$ for any $x$ (similar for $G_{\eta}$), as well as the boundedness of the output range of $G$, i.e. that there exists $\Delta$ such that $\max\{\|G_\eta(x)\|,\|G_\eta(x)\|_\infty\}\leq\Delta$ for any input $x$.
To characterize the mixing conditions, we  assume there exists a universal function $\beta$, such that $\max\{\bar{\eta}_{i,j}(\{(r_i,x_i)\}_{i=1}^T),\bar{\eta}_{i,j}(\{x_i\}_{i=1}^T)\}\leq\beta(|j-i|)),$ 
%
%
and with $\Delta_\beta=\sum_{k=1}^\infty\beta(k)<\infty$, where $\beta$'s are the {\em  mixing coefficients}.
Lastly, and as holds for the Wasserstein GAN, there exists a constant $\tilde{L}$, such that $\|D_\theta(x)-D_\theta(x')\|\leq \tilde{L}\|x-x'\|$ for all $\theta$. 
We first claim the boundedness of the decision-related quantities in DAT-CGAN. We defer the proofs of Lemma~\ref{lem:1} and Theorem~\ref{thm:1} to Appendix.
\begin{lemma}[Boundedness of decision-related quantities]\label{lm:boundedness}
Under the assumptions above, the decision-related quantities we considered are all bounded, where the bounds are universal and only depend on $B_r$. 
\label{lem:1}
\end{lemma}

Let $B_f$  denote the bound of the decision-related quantity,  $\max\{\|f(R_{t_i+1:t_i+K},x_{t_i})\|,\|f(R'_{t_i+1:t_i+K},x_{t_i})\|\}\leq B_f$ for all $i$. By
Lemma \ref{lm:boundedness}, we obtain the following generalization bound for $|\mathcal{D}_\Theta(\hat{\mathcal{P}}_R(I),\hat{\mathcal{P}}_{G_\eta,Z}(I)))-\mathcal{D}_\Theta(\mathcal{P}_R,\mathcal{P}_{G_\eta,Z})|$, for each iteration of the training process (referring to each round of the mix-max optimization of CGANs).
\begin{theorem}
\label{thm:1}
Under the assumptions above, suppose $G_{\eta_1},G_{\eta_2},\cdots,G_{\eta_{M}}$ be the $M$ generators in the $M$ iterations of the training, let $B_*=\sqrt{B^2_{f}+B^2_x}(K+\Delta_\beta)$, then
\begin{equation*}
\sup_{j\in[M]}|{D}_\Theta(\hat{\mathcal{P}}_R(I),\hat{\mathcal{P}}_{G_{\eta_j},Z}(I));\eta)-{D}_\Theta(\mathcal{P}_R,\mathcal{P}_{G_{\eta_j},Z})|\leq \varepsilon,
\end{equation*}
with probability of at least 
$$1-C\exp\Big(p\log(\frac{pL}{\varepsilon})\Big)(1+M)\exp\Big(-\frac{I\varepsilon^2}{\tilde{L}^2B^2_*}\Big),$$
for some constant $C>0$.
\end{theorem}
Theorem~\ref{thm:1} provides, whether for raw data or one of the decision-related quantities, that the distribution on non-synthetic data is close to the generated distribution at every iteration in the training process.
As with \cite{Arora2017}, we  obtain an exponential tail bound, but in our case, this has a constant that also involves the mixing coefficient as well as the sampling block size.

\section{DAT-CGAN for Portfolio Choice}

We apply the DAT-CGAN to a portfolio choice problem, where an end user is an investor who wants to understand the properties of a portfolio strategy. 
A good simulator should not only generate synthetic asset return data but also 
support the calculation of high-fidelity, decision-related quantities that are relevant for portfolio choice. 

Specifically, we  assume the investor  formulates a {\em mean-variance portfolio optimization problem} in choosing the portfolio weights. 
The investor wants to invest across a number of assets, considering the portfolio return and  the portfolio risk.
Let $r_{t+k+1}$ denote the {\em asset return vector} at time $t+k+1$, for look-ahead step $k+1$. Let  $x_t$ denote the  conditioning variables  at time $t$. We use $w_{t+k|t}$ to denote the portfolio weights decided at time $t+k$, and traded on at time $t+k+1$.
For non-synthetic data, the  portfolio optimization problem at time $t+k$ can be written as:
\begin{align*}
\max_{\quad w_{t+k|t}^\intercal\mathbf{1} = 1} &\quad  w_{t+k|t}^\intercal \hat{u}_{t+k|t} - \phi \cdot w_{t+k|t}^\intercal \hat{\Sigma}_{t+k|t} w_{t+k|t},
\end{align*}
where $\phi>0$ is the risk preference parameter, and  with
 estimated mean and  co-variance  of asset returns,  $\hat{u}_{t+k|t}$ and
 $\hat{\Sigma}_{t+k|t}$, respectively.
\begin{figure}[h]
\centering
 \includegraphics[width = 1\textwidth]{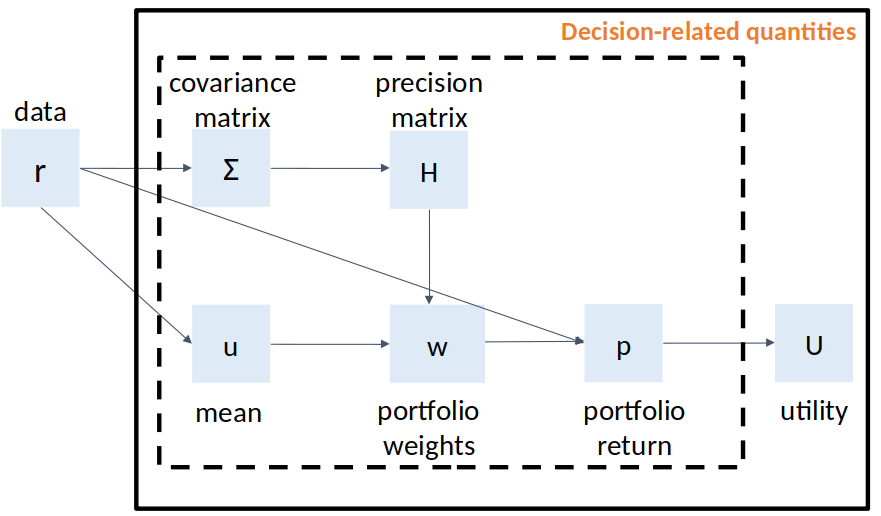}
 \caption{Decision-related quantities in the portfolio selection problem. \label{fig_dec_rel_quant}}
\end{figure}

These estimators are defined on non-synthetic asset returns as, $\hat{u}_{t+k|t}= f_{u,k}(R_{t+1:t+k}, x_t) = \mbox{MA}_{\zeta}(r_{t+k})$ and $\hat{\Sigma}_{t+k|t}  = f_{\Sigma,k}(R_{t+1:t+k}, x_t) =\mbox{MA}_{\zeta}(r_{t+k}r^\intercal_{t+k}) - \hat{u}_{t+k|t}^2.$
%
%
%
Here, $\mbox{MA}_{\zeta}(r_{t+k}) = \zeta\cdot \mbox{MA}_{\zeta}(r_{t+k-1}) + (1-\zeta)\cdot r_{t+k}$ is a moving average operator, and $\zeta>0$  a smoothing parameter.
The analytical solution to the investment problem is,
\begin{align}
w_{t+k|t} 
&=\frac{2\hat{H}_{t+k|t}}{\phi}(\hat{u}_{t+k|t} -\frac{\mathbf{1}^\intercal \hat{H}_{t+k|t}\hat{u}_{t+k|t}\mathbf{1} - 2\phi\mathbf{1}}{\mathbf{1}^\intercal \hat{H}_{t+k|t}\mathbf{1}}),
\label{portfolio_policy}
\end{align}
where $\hat{H}_{t+k|t}$ is  the    estimated precision matrix ($\hat{\Sigma}^{-1}_{t+k|t}$ ) of asset returns. $\hat{H}_{t+k|t} = ((1 - \tau) \hat{\Sigma}_{t+k|t} + \tau \Lambda)^{-1}$ using the shrinkage method~\cite{Copas1983}, where $\Lambda$ is the identity matrix  
and $\tau>0$ is a shrinkage parameter.
%
The investor is interested in the {\em realized portfolio return}, $p_{t+k+1|t} = w_{t+k|t}^\intercal r_{t+k+1|t}$, and the realized {\em utility of the portfolio return} given the risk preference,
defined as $U_{t+k+1|t} = p_{t+k+1|t} -\phi p_{t+k+1|t}^2$.
%
%
%
We give the relationship between the various
decision-related quantities
in Figure~\ref{fig_dec_rel_quant}.
These decision-related quantities are generated based on  conditioning variables
that reflect  market conditions. 
We need to take derivatives through the portfolio optimization problem during training
and handle this with the chain rule,
 making use of the closed-formed solution~\eqref{portfolio_policy}.
For the synthetic data, the entire workflow is the same as with the non-synthetic data,
except that the asset returns, $r'_{t+k|t}$, are
 generated by a GAN, where $z_{t,t+k}$ is the random seed. 
Similar to the non-synthetic data, we define $\hat{u}'_{t+k|t} = f_{u,k}(R'_{t+1:t+k}, x_t)$, $\hat{\Sigma}'_{t+k|t} = f_{\Sigma,k}(R'_{t+1:t+k}, x_t)$, $\hat{H}'_{t+k|t}$, $w'_{t+k|t}$, $p'_{t+k+1|t}$, and $U'_{t+k+1|t}$. 

\begin{figure*}[h!]
\centering
\begin{subfigure}[b]{0.32\textwidth}
\centering
\includegraphics[width = \textwidth]{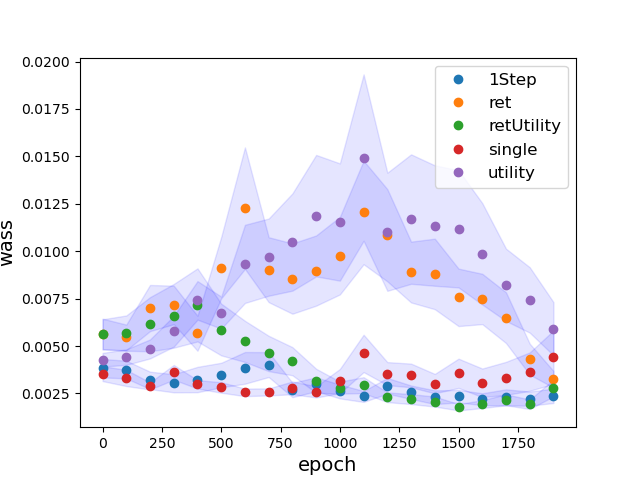}
\caption{Initial step asset returns~~~~~~~~~ \label{fig_sim_ret1}}
\end{subfigure}
\hfill
\begin{subfigure}[b]{0.32\textwidth}
\centering
\includegraphics[width = 1\textwidth]{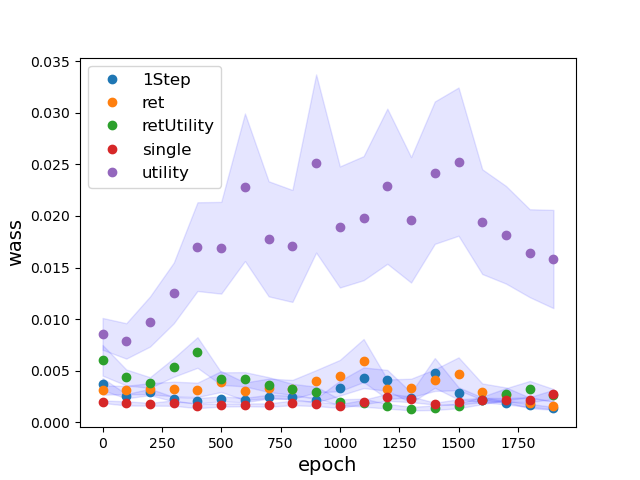}
\caption{Initial step precision matrix \label{fig_sim_invCov1}}
\end{subfigure}
\hfill
\begin{subfigure}[b]{0.32\textwidth}
\centering
\includegraphics[width = 1\textwidth]{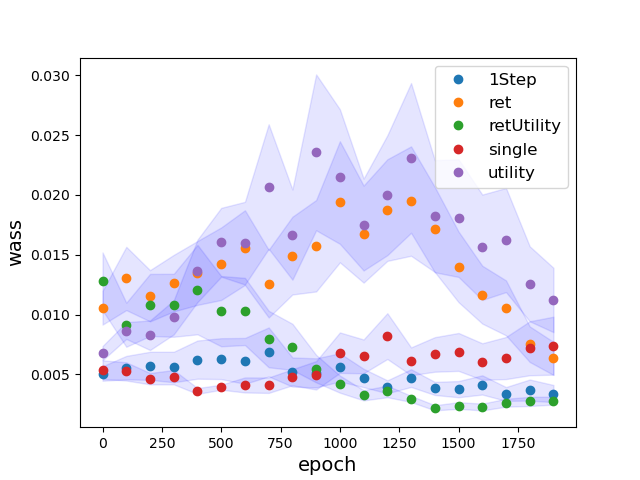}
\caption{Initial step portfolio weights \label{fig_sim_portW1}}
\end{subfigure}
\vskip\baselineskip
\begin{subfigure}[b]{0.32\textwidth}
\centering
\includegraphics[width = \textwidth]{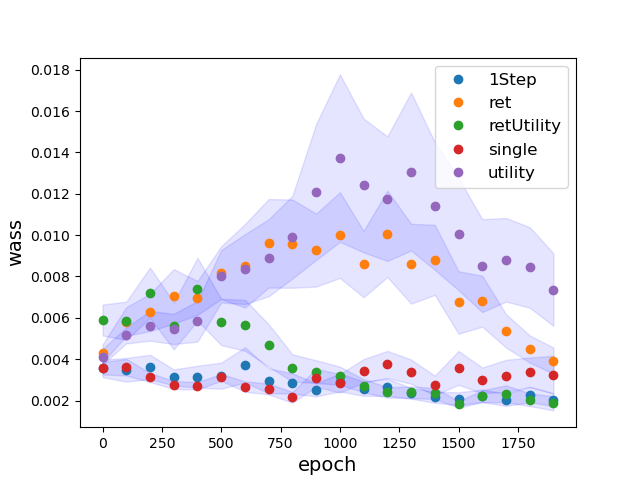}
\caption{Final step asset returns~~~~~~~~~ \label{fig_sim_ret4}}
\end{subfigure}
\hfill
\begin{subfigure}[b]{0.32\textwidth}
\centering
\includegraphics[width = 1\textwidth]{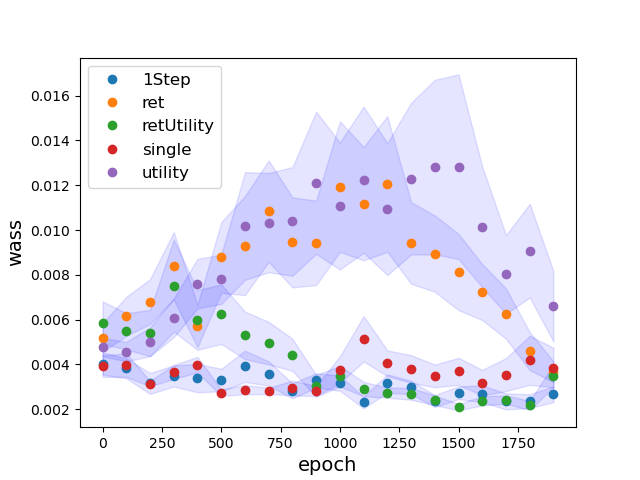}
\caption{Final step precision matrix~~~~ \label{fig_sim_invCov3}}
\end{subfigure}
\hfill
\begin{subfigure}[b]{0.32\textwidth}
\centering
\includegraphics[width = 1\textwidth]{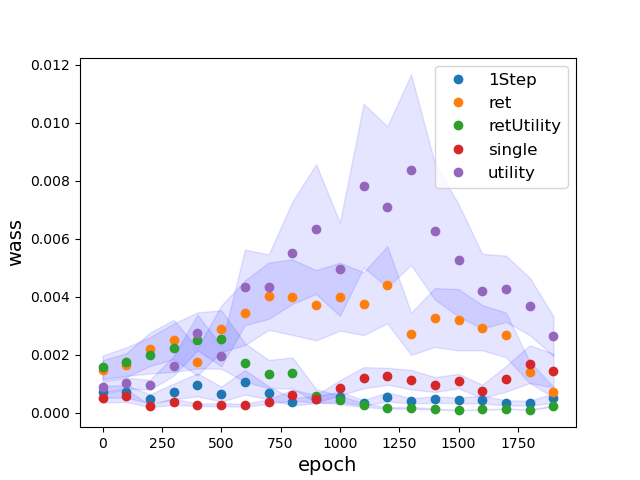}
\caption{Final step portfolio weights \label{fig_sim_portW3}}
\end{subfigure}
\caption{\label{fig_sim1}
Wasserstein Distance between non-synthetic and synthetic data for simulated time series. \ref{fig_sim_ret1} and \ref{fig_sim_ret4}  asset returns; \ref{fig_sim_invCov1} and \ref{fig_sim_invCov3}  estimated precision matrix; \ref{fig_sim_portW1} and \ref{fig_sim_portW3}  portfolio weights.  An epoch is one full pass of the data, and the shaded areas are confidence bands computed over 5 runs. }
\end{figure*}

\section{Experimental Results}

We study two different  environments. The first is a simulated environment, and the second is real and based on a basket of  ETF time series. {\em To avoid ambiguity, in this section we use the phrase ``simulated'' to refer to the simulated ground-truth model, and ``synthetic" to refer to the data generated by the DAT-CGAN and other baselines (whether in a simulated or real environment).}

\noindent\textbf{Experimental setup.} We study a $K=4$ future steps generation problem. Assume the risk-preference parameter of an investor is $\phi = 1$, and with shrinkage parameter $\tau = 0.01$ when  estimating the precision matrix (to avoid issues with a degenerate co-variance matrix). 
We use the DAT-CGAN simulator with asset returns as raw data and with the realized utility of the portfolio as the decision-related quantity. Thus, we  call this the {\em Ret-Utility-GAN}. 
We adopt utility as the decision-related quantity since it comes at the end of the decision chain and controls all the decision-related quantities;  in particular, the derivative of this quantity also involves the derivative of earlier quantities, via the chain rule, and thus controls multiple quantities of interest.
%
 We find in our experiments that this provides  good fidelity for  the  synthetic distribution  given different risk preference parameters (See Figure \ref{fig_sim1} and \ref{fig_sim2} in Section \ref{experiment appendix}). For the conditioning variables for each  asset, we use technical indicators computed via the moving average operator of asset returns in the past few weeks. (See Section \ref{conditioning variables appendix}).

We compare our method with the following approaches:

$\bullet$ ({\em Ret-GAN}) A GAN defined with  the  asset return loss, which is a standard model used~\cite{Koshiyama2019,Zhou2018}. This GAN generates  synthetic raw asset returns $R'_{t+1:t+K}$ for each $t$, and the training process uses the sum of $K$ Wasserstein losses, one for each look-ahead step.

$\bullet$ ({\em 1Step-GAN})  A 1-step version of the Ret-Utility GAN: a GAN with a 1-step look ahead asset return and utility (this is similar to~\cite{Li2019}, but with an additional decision-aware loss). This GAN generates  synthetic data $R'_{t+1:t+1}$ (i.e. $r'_{t+1|t}$) and synthetic derived quantity $U'_{t+1|t}$, i.e., utility. The training process  uses the sum of two Wasserstein losses, defined on $R'_{t+1:t+1}$ and  $U'_{t+1|t}$. 

$\bullet$ ({\em Single-GAN}) A GAN with a single Wasserstein
loss defined on a vector of $2K$ quantities coming from stacking the asset returns, $R'_{t+1:t+K}$, and utility quantities $U'_{t+k|t}$, for each $k$ and $t$.

$\bullet$ ({\em Utility-GAN}) A GAN with only the  utility loss (i.e., no loss on the asset returns). 
This GAN generates the synthetic derived quantities, $U'_{t+k|t}$, for each $k$ and $t$, i.e., the utility quantities.  The training process uses the sum of $K$ Wasserstein losses, one for each look-ahead step.  
\begin{figure*}[h!]
\centering

\begin{subfigure}[b]{0.32\textwidth}
\centering
\includegraphics[width = \textwidth]{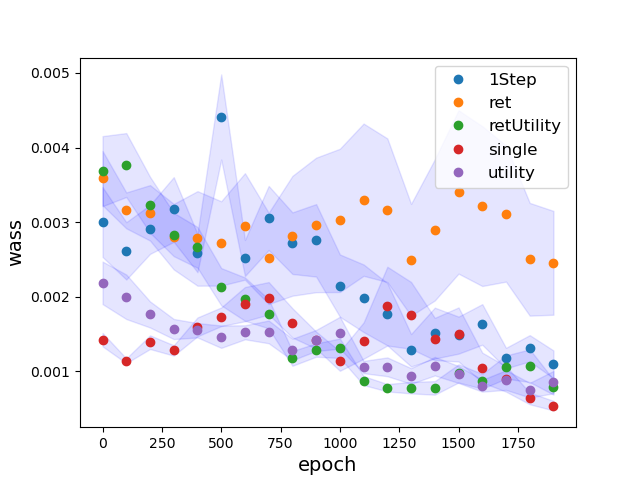}
\caption{Initial step asset returns~~~~~~~~~ \label{fig_real_ret1}}
\end{subfigure}
\hfill
\begin{subfigure}[b]{0.32\textwidth}
\centering
\includegraphics[width = 1\textwidth]{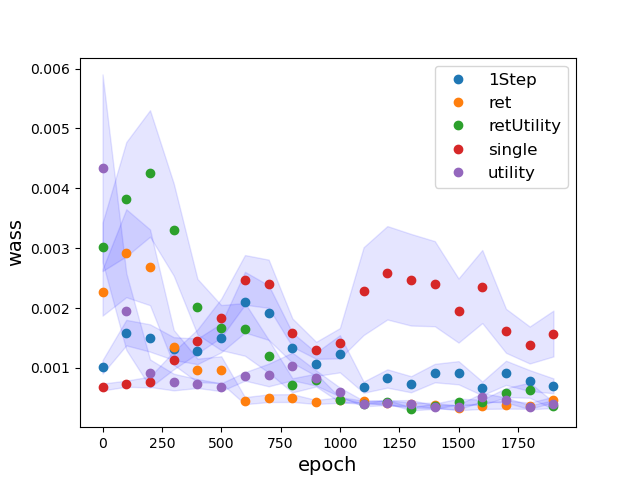}
\caption{Initial step precision matrix \label{fig_real_invCov1}}
\end{subfigure}
\hfill  
\begin{subfigure}[b]{0.32\textwidth}
\centering
\includegraphics[width = 1\textwidth]{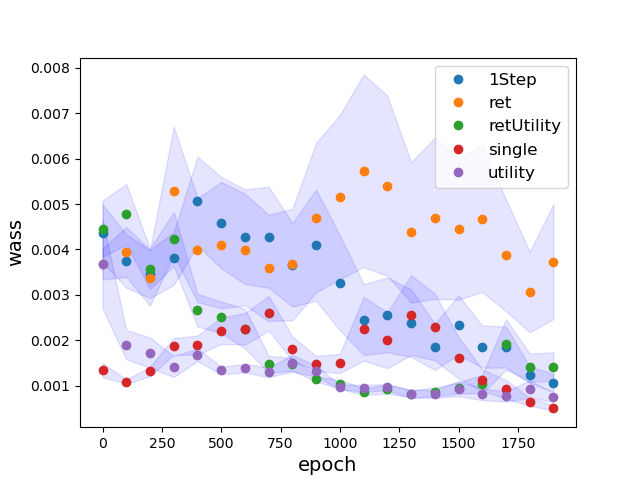}
\caption{Initial step portfolio weights \label{fig_real_portW1}}
\end{subfigure}
\vskip\baselineskip
\begin{subfigure}[b]{0.32\textwidth}
\centering
\includegraphics[width = \textwidth]{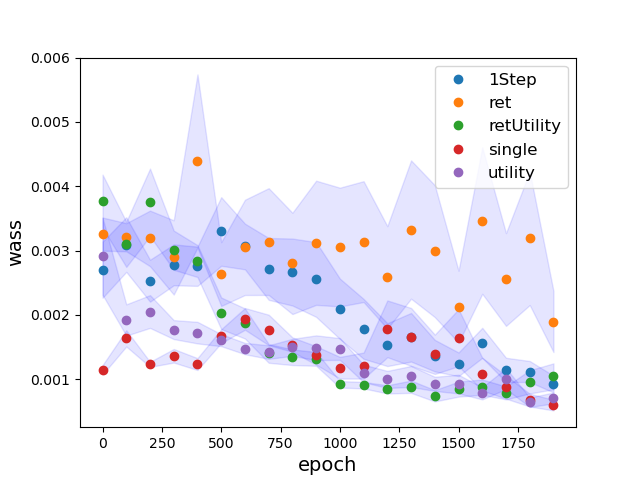}
\caption{Final step asset returns~~~~~~~~~ \label{fig_real_ret4}}
\end{subfigure}
\hfill
\begin{subfigure}[b]{0.32\textwidth}
\centering
\includegraphics[width = 1\textwidth]{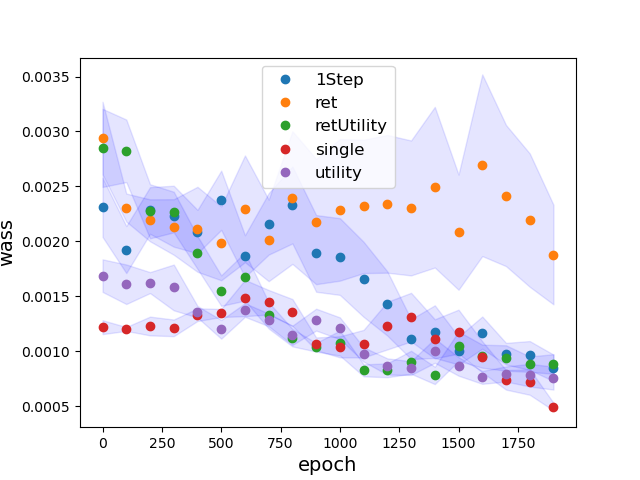}
\caption{Final step precision matrix~~~\label{fig_real_invCov3}}
\end{subfigure}
\hfill
\begin{subfigure}[b]{0.32\textwidth}
\centering
\includegraphics[width = 1\textwidth]{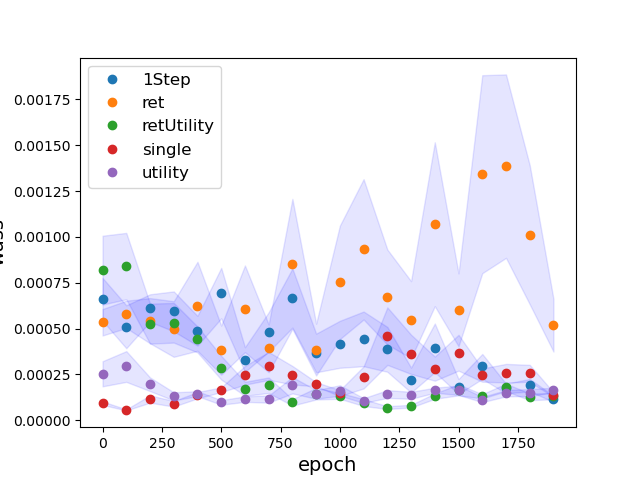}
\caption{Final step portfolio weights \label{fig_real_portW3}}
\end{subfigure}
\caption{Wasserstein Distances between non-synthetic and synthetic data for four U.S. ETFs. \ref{fig_real_ret1} and \ref{fig_real_ret4} for asset returns; \ref{fig_real_invCov1} and \ref{fig_real_invCov3}  for estimated precision matrix; \ref{fig_real_portW1} and \ref{fig_real_portW3} for portfolio weights. An epoch is one full pass of the data, and the shaded areas are confidence bands computed over 5 runs.  \label{fig_real1}}
\end{figure*}

For the generator, we use a two-layer feed-forward neural network for each asset. The output are the asset returns, and these are  used to compute any decision-related quantities. Where necessary, we make use of multiple discriminators, each  corresponding to a particular quantity (e.g., raw data, or a decision-related quantity). For each discriminator, the architecture is a two-layer feed-forward neural network(See Section \ref{appendix generator network architecture} and \ref{appendix discriminator network architecture}). For evaluation, we calculate the Wasserstein distances between the raw asset returns, the estimated precision matrix, and the portfolio weights (the decision variables) and their respective synthetic counterparts. We do this for 
quantities that correspond to each of 
 the initial and  final lookahead steps
 (i.e., steps 1 and 3 for the estimated precision matrix and portfolio weights, and steps 1 and 4 for the raw data).

\noindent\textbf{Results: Simulated time series.}
\label{sim experiment}
We first present results on a simulated time series. The data-generating process is given by  $r_{t+1} = b_0 \cdot r_t + \sum^{4}_{i=1} b_i \cdot \mbox{MA}_{\zeta_i}(r_t) + \epsilon$, where $r_{t}$ is the asset return vector, $\mbox{MA}_{\zeta_i}(r_t)$ the moving average operator, $\zeta_i>0$ the smoothing parameter, $b_i$ the coefficient, and $\epsilon$ the noise. We use a {\em multivariate t-distribution} to model the noise, with parameters $\mu$, $\sigma$ (See Section \ref{appendix sim data generating process}) and d.o.f.,~$\nu = 100$. It simulates the heavy-tail behavior of  asset returns~\cite{Bollerslev1987}.

Figures~\ref{fig_sim_ret1} to~\ref{fig_sim_portW3} confirm that the Ret-Utility-GAN has the best performance in terms of minimizing the Wasserstein distance for each of asset returns, estimated precision matrix, and portfolio weights. It performs (1) better than Ret-GAN, confirming  that introducing the decision-related quantity, utility, provides useful moderation on the distribution of synthetic raw data; (2) better than the Utility-GAN, which shows that including the asset return loss also helps; (3) better than the  Single-GAN, which shows that imposing loss for each quantity is more effective than a single loss  on stacked quantities; and 
(4) better than the 1step-GAN, which shows that the Ret-Utility-GAN is effective in addressing exposure bias. We can also observe that Ret-GAN doesn't perform as well as the Ret-Utility-GAN even on the raw asset return data. Our hypothesis is that, by introducing additional loss, Ret-Utility-GAN provides more information for the gradient during the training process, leading to more effective training stability and better generalization on evaluation data. 

\noindent\textbf{Results: Real ETF time series.}
%
We also present results on a real ETF time series, where we use weekly price data for each of four U.S.~ETFs\footnote{the {\em Material (XLB)}, {\em Energy (XLE)}, {\em Financial (XLF)}, and {\em Industrial (XLI)} ETFs.}
from 1999 to 2016. The data includes the end-of-day price  for each ETF. The entire dataset has more than 3500 data points (17 years $\times$ 52 weeks $\times$ 4 ETFs). We divide the data into a training set with data from 1999--2006, and an evaluation set with data from 2007--2016. As a reference,~\cite{Yoon2019} make use of 4,000 data points for their experiment, thus our sample size is comparable. We generate future weekly returns for each of $K=4$ future weeks (1 month).

 Figures~\ref{fig_real_ret1} through \ref{fig_real_portW3} show that the performance of  Ret-Utility-GAN is much better than Ret-GAN in regard to the Wasserstein loss 
in regard to each of
 the asset returns, precision matrix, and  portfolio weights. Ret-Utility-GAN is also better than the Single-GAN in regard to the Wasserstein loss on the initial step precision matrix, and better than the 1Step-GAN in terms of its  training stability. In this experiment, Utility-GAN, which also makes use of a decision-related quantity, also performs well. We again observe that Ret-GAN does not perform as well as the Ret-Utility-GAN, even on the raw asset return data.

\section{Conclusion}
We proposed DAT-CGAN,
 a novel, decision-aware time series conditional generative adversarial network for generating  time-series data. The  method incorporates decision-related quantities into 
 a multi-loss structure, avoids  exposure bias  by aligning look-ahead steps during training and testing, and alleviates problems with data scarcity through an overlapped-block sampling scheme. Moreover, we characterize the generalization properties of DAT-CGANs for generating the raw data as well as decision-related quantities.
In an application to portfolio selection, we demonstrated better generative quality for decision-related quantities, such as estimated precision matrix and portfolio weights, than other strong, GAN-based baselines.

\newpage

\bibliographystyle{alpha}
\bibliography{DAT-CGAN.bib}

\newcommand{\etalchar}[1]{$^{#1}$}
\begin{thebibliography}{RMM{\etalchar{+}}17}

\bibitem[ACB17]{Arjovsky2017}
Mart{\'{\i}}n Arjovsky, Soumith Chintala, and L{\'{e}}on Bottou.
\newblock Wasserstein generative adversarial networks.
\newblock In {\em Proceedings of the 34th International Conference on Machine
  Learning}, volume~70 of {\em Proceedings of Machine Learning Research}, pages
  214--223. {PMLR}, 2017.

\bibitem[AGL{\etalchar{+}}17]{Arora2017}
Sanjeev Arora, Rong Ge, Yingyu Liang, Tengyu Ma, and Yi~Zhang.
\newblock Generalization and equilibrium in generative adversarial nets (gans).
\newblock In {\em Proceedings of the 34th International Conference on Machine
  Learning}, volume~70 of {\em Proceedings of Machine Learning Research}, pages
  224--232. {PMLR}, 2017.

\bibitem[AIMM19]{Athey2019}
Susan Athey, Guido Imbens, Jonas Metzger, and Evan Munro.
\newblock Using {Wasserstein} generative adversarial networks for the design of
  monte carlo simulations.
\newblock {\em arXiv:1909.02210}, 2019.

\bibitem[B\"02]{Buhlmann2015}
Peter B\"{u}hlmann.
\newblock Bootstraps for time series.
\newblock {\em Statistical Science}, 17:52--72, 2002.

\bibitem[BN20]{Bashar2020}
MA~Bashar and R~Nayak.
\newblock Tanogan: Time series anomaly detection with generative adversarial
  networks.
\newblock In {\em 2020 IEEE Symposium Series on Computational Intelligence
  (SSCI)}, 2020.

\bibitem[Bol87]{Bollerslev1987}
Bollerslev.
\newblock A conditional heteroskedastic time series model for speculative
  prices and rates of return.
\newblock {\em Review of Economics and Statistics}, 69:542--547, 1987.

\bibitem[Bra07]{Bradley2008}
R.C. Bradley.
\newblock Introduction to strong mixing conditions.
\newblock In {\em Kendrick Press, Heber City (Utah)}, 2007.

\bibitem[BVJS15]{Bengio2015}
Samy Bengio, Oriol Vinyals, Navdeep Jaitly, and Noam Shazeer.
\newblock Scheduled sampling for sequence prediction with recurrent neural
  networks.
\newblock In {\em Advances in Neural Information Processing Systems 28: Annual
  Conference on Neural Information Processing Systems}, pages 1171--1179, 2015.

\bibitem[Cop83]{Copas1983}
J.B. Copas.
\newblock Regression, prediction and shrinkage.
\newblock In {\em Journal of the Royal Statistical Society}, 1983.

\bibitem[GLA20]{Geiger2020}
A~Geiger, D~Liu, and S~Alnegheimish.
\newblock Tadgan: Time series anomaly detection using generative adversarial
  networks.
\newblock In {\em 2020 IEEE International Conference on Big Data}, 2020.

\bibitem[HN20]{Ni2020}
Magnus~Wiese Hao~Ni, Lukasz~Szpruch.
\newblock Conditional sig-wasserstein gans for time series generation.
\newblock In {\em arxiv}, 2020.

\bibitem[KFT19]{Koshiyama2019}
Adriano~Soares Koshiyama, Nick Firoozye, and Philip~C. Treleaven.
\newblock Generative adversarial networks for financial trading strategies
  fine-tuning and combination.
\newblock {\em CoRR}, abs/1901.01751, 2019.

\bibitem[KR08]{kontorovich2008concentration}
Leonid~Aryeh Kontorovich and Kavita Ramanan.
\newblock Concentration inequalities for dependent random variables via the
  martingale method.
\newblock {\em The Annals of Probability}, 36(6):2126--2158, 2008.

\bibitem[LCG18]{Li2018}
D~Li, D~Chen, and J~Goh.
\newblock Anomaly detection with generative adversarial networks for
  multivariate time series.
\newblock In {\em arxiv}, 2018.

\bibitem[LCJ19]{Li2019b}
D~Li, D~Chen, and B~Jin.
\newblock Mad-gan: Multivariate anomaly detection for time series data with
  generative adversarial networks.
\newblock In {\em Advances in Neural Information Processing Systems 31}, 2019.

\bibitem[LCZ18]{Luo2018}
Y~Luo, X~Cai, and Y~Zhang.
\newblock Multivariate time series imputation with generative adversarial
  networks.
\newblock In {\em Advances in Neural Information Processing Systems 31}, 2018.

\bibitem[LWL20]{Li2019}
Junyi Li, Xintong Wang, and Yaoyang Lin.
\newblock Generating realistic stock market order streams.
\newblock In {\em The Thirty-Fourth {AAAI} Conference on Artificial
  Intelligence}, pages 727--734, 2020.

\bibitem[LZC19]{Luo2019}
Y~Luo, Y~Zhang, and X~Cai.
\newblock E2gan: End-to-end generative adversarial network for multivariate
  time series imputation.
\newblock In {\em Proceedings of the Twenty-Eighth International Joint
  Conference on Artificial Intelligence (IJCAI-19)}, 2019.

\bibitem[Mar52]{Markowitz1952}
Harry Markowitz.
\newblock Portfolio selection.
\newblock In {\em The Journal of Finance}, 1952.

\bibitem[MO14]{Mirza2014}
Mehdi Mirza and Simon Osindero.
\newblock Conditional generative adversarial nets.
\newblock {\em CoRR}, abs/1411.1784, 2014.

\bibitem[RCAZ16]{Ranzato15}
Marc'Aurelio Ranzato, Sumit Chopra, Michael Auli, and Wojciech Zaremba.
\newblock Sequence level training with recurrent neural networks.
\newblock In {\em 4th International Conference on Learning Representations},
  2016.

\bibitem[RMM{\etalchar{+}}17]{Rennie2017}
Steven~J. Rennie, Etienne Marcheret, Youssef Mroueh, Jerret Ross, and Vaibhava
  Goel.
\newblock Self-critical sequence training for image captioning.
\newblock In {\em 2017 {IEEE} Conference on Computer Vision and Pattern
  Recognition}, pages 1179--1195, 2017.

\bibitem[TCTI19]{Takahashi2019}
S~Takahashi, Y~Chen, and K~Tanaka-Ishii.
\newblock Modeling financial time-series with generative adversarial networks.
\newblock In {\em Physica A: Statistical Mechanics and its Applications}, 2019.

\bibitem[Vil09]{Villani2009}
C\'{e}dric Villani.
\newblock Optimal transport: Old and new.
\newblock In {\em Springer, Berlin}, 2009.

\bibitem[WKK19]{Wiese2019}
Magnus Wiese, Robert Knobloch, and Ralf Korn1.
\newblock Quant gans: Deep generation of financial time series.
\newblock In {\em arxiv}, 2019.

\bibitem[YJvdS19]{Yoon2019}
Jinsung Yoon, Daniel Jarrett, and Mihaela van~der Schaar.
\newblock Time-series generative adversarial networks.
\newblock In {\em Advances in Neural Information Processing Systems 32: Annual
  Conference on Neural Information Processing Systems}, pages 5509--5519, 2019.

\bibitem[ZPH{\etalchar{+}}18]{Zhou2018}
Xingyu Zhou, Zhisong Pan, Guyu Hu, Siqi Tang, and Cheng Zhao.
\newblock Stock market prediction on high-frequency data using generative
  adversarial nets.
\newblock {\em Mathematical Problems in Engineering}, 2018.

\end{thebibliography}

\newpage
\appendix
\section{Omitted Proofs}

\subsection{Proof of Lemma 1}

We first show the decision-related quantities  in DAT-CGAN are all bounded under our assumptions. Here by boundedness of a vector or matrix, we mean the largest entry of the vector or matrix is bounded. Similar to the theory part in the main paper, we also simplify notations, such as only considering fixed $k=K-1$ for estimated mean $\hat{u}_{t+K-1|t}$, estimated covariance matrix  $\hat{\Sigma}_{t+K-1|t}$, estimated precision matrix $\hat{H}_{t+K-1|t}$ and portfolio weights $\hat{w}_{t+K-1|t}$; only considering fixed $k=K$ for realized portfolio return $p_{t+K|t}$ and utility $U_{t+K|t}$. We also omit some constants. That will affect the validity of the application of our theory to the algorithm.

The quantities of interest are:
\begin{itemize}
\item[a.] $\hat{u}_{t+K-1|t}=\mbox{MA}_{\zeta}(r_{t+K-1})$, where $\mbox{MA}_{\zeta}(r_{t+K-1}) = \zeta\mbox{MA}_{\zeta}(r_{t+K-2}) + (1-\zeta) r_{t+K-1}$ and $0 < \zeta < 1$;
\item[b.]$\hat{\Sigma}_{t+K-1|t}=\mbox{MA}_{\zeta}(r_{t+K-1}r^{\intercal}_{t+K-1}) - \hat{u}_{t+K-1|t}^2$;
\item[c.]$\hat{H}_{t+K-1|t} = ((1 - \tau) \hat{\Sigma}_{t+K-1|t} + \tau \Lambda)^{-1}$, where $0 < \tau < 1$ and $\Lambda$ is the identity matrix;
\item[d.]$$\hat{w}_{t+K-1|t}=h(\hat{u}_{t+K-1|t}, \hat{H}_{t+K-1|t})=\frac{2\hat{H}_{t+K-1|t}}{\phi}\left(\hat{u}_{t+K-1|t}-\frac{{\bm 1}^\intercal\hat{H}_{t+K-1|t}\hat{u}_{t+K-1|t} \bm 1-2\phi \bm 1}{{\bm 1}^\intercal \hat{H}_{t+K-1|t} \bm 1}\right);$$
\item[e.] $p_{t+K|t} = \hat{w}_{t+K-1|t}^\intercal r_{t+K}$ 
\item[f.] $U(\hat{w}_{t+K-1|t}^\intercal r_{t+K})=\hat{w}_{t+K-1|t}^\intercal r_{t+K}-\phi (\hat{w}_{t+K-1|t}^\intercal r_{t+K})^2$ 
\end{itemize}

a, b are obviously bounded since $r_t$ is bounded, $\forall t$. e, f are also bounded once we prove c, d are bounded.
\begin{itemize}
\item[I.] For $\hat{H}_{t+K-1|t}$, we can obtain its boundedness by simply realizing the determinant of $\hat{H}_{t+K-1|t}$ is lower bounded, and that the determinant of the adjugate matrices are all upper bounded, and applying Cramer's rule.

The lower bound of $|\hat{H}^{-1}_{t+K-1|t}|$ can be obtained by
\begin{align*}
|\hat{H}^{-1}_{t+K-1|t}|&=\Pi_{j=1}^d((1-\tau)\tau_i+\tau)\geq \tau^d,
\end{align*}
where $\tau_i$'s are eigenvalues of matrix $\hat{\Sigma}_{t+K-1|t}=\mbox{MA}_{\zeta}(r_{t+K-1}r^{\intercal}_{t+K-1}) - \hat{u}_{t+K-1|t}^2$, which are all non-negative.
The upper bound of  the determinant  of the $(d-1)\times (d-1)$ adjugate matrices are clearly upper bounded since every entry of them is bounded.

\item[II.] Next, we consider 
$$\hat{w}_{t+K-1|t}=h(\hat{u}_{t+K-1|t}, \hat{H}_{t+K-1|t})=\frac{2\hat{H}_{t+K-1|t}}{\phi}\left(\hat{u}_{t+K-1|t}-\frac{{\bm 1}^\intercal\hat{H}_{t+K-1|t}\hat{u}_{t+K-1|t} \bm 1-2\phi \bm 1}{{\bm 1}^\intercal \hat{H}_{t+K-1|t} \bm 1}\right);$$
Notice $\tau_{\max}(\hat{H}_{t+K-1|t})\leq 1/\tau_{\min}(\hat{H}^{-1}_{t+K-1|t})\leq 1/\tau$, for any vector $v$
$$|\hat{H}_{t+K-1|t}v|\leq \frac{\|v\|}{\tau}.$$
Besides, $\tau_{\max}(\hat{H}^{-1}_{t+K-1|t})\leq \|\hat{H}^{-1}_{t+K-1|t}\|_F$. The Frobenius norm of matrix $\hat{H}^{-1}_{t+K-1|t}$ is bounded since every entry in the matrix is bounded. Thus, for any $v$,
$$|\hat{H}_{t+K-1|t}v|\geq \frac{\|v\|}{\|\hat{H}^{-1}_{t+K-1|t}\|_F}\geq \frac{\|v\|}{\sqrt{(1-\tau)^2 B^2_r+\tau^2} d}.$$
Then, we know $\hat{w}_{t+K-1|t}$ is bounded. Notice all the bounds mentioned above  can be obtained by only using $B_r$, $\tau$, $d$, and $\phi$.

\end{itemize}

\subsection{Proof of Theorem 1}

For convenience of statement, we  restate the mixing condition framework and corresponding lemma. 

\paragraph{Restatement of Result in \cite{kontorovich2008concentration}} We  consider a simplified variant of Theorem 1.1 in \cite{kontorovich2008concentration}. Let $X_i\in S$, where $S$ is a finite set, and $X=(X_1,X_2,\cdots,X_n)$. We further denote $X^j_i=(X_i,X_{i+1},\cdots,X_j)$ as a random vector for  $1\leq i< j\leq n$.  Correspondingly, we let  $x^j_i=(x_i,x_{i+1},\cdots,x_j)$ be a subsequence for $(x_1,x_2,\cdots,x_n)$. And let 
$$\bar{\eta}_{i,j}=\sup_{y^{i-1}_1\in S^{i-1}, w,w'\in S,~\bP(X^i_1=Y^{i-1}w)>0,~\bP(X^i_1=Y^{i-1}w')>0}\eta_{i,j}(y^{i-1}_1,w,w'),$$
where 
$$\eta_{i,j}(y^{i-1}_1,w,w')=TV\Big(\cD(X^n_j|X^i_1=y^{i-1}_1w),\cD(X^n_j|X^i_1=y^{i-1}_1w')\Big).$$
Here $TV$ is the total variation distance, and $\cD(X^n_j|X^i_1=y^i_1w)$ is the conditional distribution of $\cD(X^n_j|X^i_1=y^i_1w)$, conditioning on $\{X^i_1=y^i_1w\}$.

Let $H_n$ be am $n\times n$ upper triangular matrix, defined by
$$(H_n)_{ij}=\left\{\begin{matrix}
 1&i=j   \\ 
 \bar{\eta}_{i,j}&i<j   \\ 
 0&   o.w.
\end{matrix}\right.$$
Then, 
$$\|H_n\|_\infty=\max_{1\leq i\leq n}J_{n,i},$$
where 
$$J_{n,i}=1+\bar{\eta}_{i,i+1}+\cdots+\bar{\eta}_{i,n},$$
and $J_{n,n}=1$.

\begin{lemma}[Variant of Result in \cite{kontorovich2008concentration}]\label{thm:konto}
Let $h$ be a $L_h$-Lipschitz function (with respect to the Hamming distance) on $S^n$ for some constant $L_f>0$. Then, for any $t>0$,
$$\bP(|h(X)-Eh(x)|\geq t)\leq 2\exp\Big(-\frac{t^2}{2nL^2_h\|H_n\|^2_\infty}\Big).$$

\end{lemma}
\bigskip

Recall under our assumptions, we have for all $\theta\in \Theta$,
$$D_\theta\big(f(R_{t_i+1:t_i+K},x_t),x_t\big)\in[L^{f}_R,U^{f}_R],$$
such that $U^{f}_R-L^{f}_R\leq 2\tilde{L}\sqrt{B^2_{f}+B^2_x}.$
Since Lemma 2 needs finite support for $S$, we take a detour here in order to extend it to an interval support. We define a $\varepsilon$-net for the interval $[L^{f}_R,U^{f}_R]$, i.e. $P_\varepsilon=\{p_0,p_1,\cdots,p_W\}$ such that $p_0=L^{f}_R,p_1=L^{f}_R+\varepsilon,\cdots p_W=L^{f}_R+W\varepsilon$ and $|U^{f}_R-p_W|\leq \varepsilon$. We define a function $g_{P_\varepsilon}(\cdot)$ on $[L^{f}_R,U^{f}_R]$ such that, for any $x\in [L^{f}_R,U^{f}_R]$, we have
$$g_{P_\varepsilon}(x)=\argmin_{p_i\in P_\varepsilon}|x-p_i|.$$

Without loss of generality, we can assume, for all $j\in\{0,\cdots,W\}$,
$$\bP(g_{P_\varepsilon}[D_\theta\big(f(R_{t_i+1:t_i+k},x_t),x_t\big)]=p_j)>0,$$ 
otherwise, we can remove the corresponding $p_j$ and form a new net.
From the mixing condition on $\{(r_{i+1},x_i)\}_{i=1}^T$, we can obtain the mixing condition on overlapping blocks.

\begin{lemma}
Under our assumptions, denote $\tilde{R}^x_{t_i+1:t_i+K}=((r_{t_i+1},x_{t_i}),\cdots,(r_{t_i+K},x_{t_i+K-1}))$, we have 
$$\bar{\eta}_{i,j}(\{\tilde{R}^x_{t_m+1:t_m+K}\}_{m=1}^I)\leq\left\{\begin{matrix}
 1&|i-j|\leq K-1  \\ 
\bar{\eta}_{i+K-1,j}((\{(r_{i+1},x_i)\}_{i=1}^T)\leq \beta(|j-i-K+1|)&o.w.   
\end{matrix}\right.$$
\label{lem:app3}
\end{lemma}
\begin{proof}
This can be immediately obtained once we know for any output range $O$, that  $(\tilde{R}^x_{i,i+K},\tilde{R}^x_{i,i+K})\in O$ is equivalent to $((r_{i+1},x_{i}),\cdots,(r_{i+K-1},x_{i+K-2}),(r_{i+K},x_{i+K-1}))\in O'$, for some output range $O'$ and $|t_i-t_{j}|\geq |i-j|$.
\end{proof}

\begin{lemma}
Under our assumptions, denote  $\tilde{Z}^x_{t_i,t_i+K}=((z_{t_i,t_i + 1},x_{t_i}),\cdots,(z_{t_i,t_i + K},x_{t_i}))$, we have
$$\bar{\eta}_{i,j}(\{\tilde{Z}^x_{t_m,K}\}_{m=1}^I)\leq\left\{\begin{matrix}
 1&|i-j|\leq K-1  \\ 
\bar{\eta}_{i+K-1,j}((\{(z_i,x_i)\}_{i=1}^T)\leq \beta(|j-i-K+1|)&o.w.   
\end{matrix}\right.$$\end{lemma}
\begin{proof}
Notice unlike $R_{t_i+1:t_i+K}$, each $\{Z_{t_i,t_i+K}\}_i$ are mutually independent, and elements in each $Z_{t_i,t_i+K}$ are also mutually independent. Thus, the mixing coefficients depend entirely on $x$. Similarly as  with Lemma~\ref{lem:app3}, we  immediately obtain the result.
\end{proof}

Then, we can use Lemma~\ref{thm:konto} to obtain the following theorem.

\begin{theorem}
With overlapping block sampling,  then for any $\varepsilon>0$, and any $\theta\in\Theta$, we have 
$$\bP(\Big|\frac{1}{I}\sum_{i=1}^{I} D_\theta\big(f(R_{t_i+1:t_i+K},x_{t_i}),x_{t_i}\big)-\bE D_\theta\big(f(R_{t_i+1:t_i+K},x_{t_i}),x_{t_i}\big)\Big|\geq 3\varepsilon)\leq 2\exp\Big(-\frac{I\varepsilon^2}{4\tilde{L}^2(B^2_{f}+B^2_x)(K+\Delta_\beta)^2}\Big).$$
\end{theorem}
\begin{proof}
By Lemma 3, and combined with the assumption that $\sum_i\beta(|i|)\leq \Delta_\beta$, with simple calculation, we can obtain for $\{\tilde{R}^x_{t_m+1:t_m+K}\}_{m=1}^I$
$$\|H_I\|_\infty\leq K+\Delta_\beta.$$
Besides, we know 
$$\frac{1}{n}\sum_{i=1}^n g_{P_\varepsilon}[D_\theta\big((f(R_{t_i+1:t_i+K},x_{t_i}),x_{t_i}\big)]$$
is $2\tilde{L}\sqrt{B^2_{f}+B^2_x}$--Lipschitz continuous with respect to the Hamming distance. Then, by Lemma \ref{thm:konto}, we have 
$$\bP(\Big|\frac{1}{I}\sum_{i=1}^{I} g_{P_\varepsilon}[D_\theta\big(f(R_{t_i+1:t_i+K},x_{t_i}),x_{t_i}\big)]-\bE g_{P_\varepsilon}[D_\theta\big(f(R_{t_i+1:t_i+K},x_{t_i}),x_{t_i}\big)]\Big|\geq \varepsilon)\leq 2\exp\Big(-\frac{I\varepsilon^2}{4\tilde{L}^2(B^2_{f}+B^2_x)(K+\Delta_\beta)^2}\Big).$$
Next, it is easy to see for any $x\in[L^{f}_R,U^{f}_R]$, we have
$$\Big| g_{P_\varepsilon}[x]-x\Big|\leq \varepsilon.$$
Thus, we can obtain 
$$\bP(\Big|\frac{1}{I}\sum_{i=1}^{I} D_\theta\big(f(R_{t_i+1:t_i+K},x_{t_i}),x_{t_i}\big)-\bE D_\theta\big(f(R_{t_i+1:t_i+K},x_{t_i}),x_{t_i}\big)\Big|\geq 3\varepsilon)\leq 2\exp\Big(-\frac{I\varepsilon^2}{4\tilde{L}^2(B^2_{f}+B^2_x)(K+\Delta_\beta)^2}\Big).$$

\end{proof}

Similarly, we have:
\begin{theorem}
With overlapping block sampling, then for any $\varepsilon>0$, and any $\theta\in\Theta$, any $\eta\in\Xi$, we have
$$\bP(\Big|\frac{1}{I}\sum_{i=1}^{I} D_\theta\big(G_\eta(Z_{t_i,t_i+K},x_{t_i}),x_{t_i}\big)-\bE D_\theta\big(G_\eta(Z_{t_i,t_i+K},x_{t_i}),x_{t_i}\big)\Big|\geq \varepsilon)\leq 2\exp\Big(-\frac{I\varepsilon^2}{4\tilde{L}^2(B^2_{f}+B^2_x)(K+\Delta_\beta)^2}\Big),$$
\end{theorem}

Now, let us consider the generalization bound under the neural-network distance for a fixed generator.
\begin{lemma}
Under the assumptions in subsection ``Assumptions and implications", there exists a universal constant $C$ such that 
\begin{equation*}
\left|\mathscr{D}_\Theta\Big(\hat{\cP}_R(I),\hat{\cP}_{G_\eta,Z}(I)\Big)-\mathscr{D}_\Theta\Big(\cP_{R},\cP_{G_\eta,Z}\Big)\right|\leq \varepsilon,
\end{equation*}
with probability
$$1-C\exp\Big(p\log(\frac{pL}{\varepsilon})\Big)\left[\exp\Big(-\frac{I\varepsilon^2}{\tilde{L}^2(B^2_{f}+B^2_x)(K+\Delta_\beta)^2}\Big)\right].$$
\end{lemma}
\begin{proof}
Recall for  a fixed $\theta$, we have the following two conditions:
$$\bP(\Big|\frac{1}{I}\sum_{i=1}^{I} D_\theta\big(f(R_{t_i+1:t_i+K},x_{t_i}),x_{t_i}\big)-\bE D_\theta\big(f(R_{t_i+1:t_i+K},x_{t_i}),x_{t_i}\big)\Big|\geq \varepsilon/4)\leq 2\exp\Big(-\frac{I\varepsilon^2}{576\tilde{L}^2(B^2_{f}+B^2_x)(K+\Delta_\beta)^2}\Big)$$
$$\bP(\Big|\frac{1}{I}\sum_{i=1}^{I} D_\theta\big(G_\eta(Z_{t_i,t_i+K},x_{t_i}),x_{t_i}\big)-\bE D_\theta\big(G_\eta(Z_{t_i,t_i+K},x_{t_i}),x_{t_i}\big)\Big|\geq \varepsilon/4)\leq 2\exp\Big(-\frac{I\varepsilon^2}{576\tilde{L}^2(B^2_{f}+B^2_x)(K+\Delta_\beta)^2}\Big)$$

Let $\cN_\Theta$ be an $\varepsilon/8L$-net of $\Theta$, which is a standard construction satisfying $\log|\cN_\Theta|\leq O(p\log(pL/\varepsilon))$, then, by the uniform bound, we can obtain
\begin{align*}
    \bP(\sup_{\theta\in\Theta}\Big|\frac{1}{I}\sum_{i=1}^{I} D_\theta\big(f(R_{t_i+1:t_i+K},x_{t_i}),x_{t_i}\big)&-\bE D_\theta\big(f(R_{t_i+1:t_i+K},x_{t_i}),x_{t_i}\big)\Big|\geq \varepsilon/2)\\ 
    & \leq 2|\cN_\Theta|\exp\Big(-\frac{I\varepsilon^2}{576\tilde{L}^2(B^2_{f}+B^2_x)(K+\Delta_\beta)^2}\Big)
    \end{align*}

\begin{align*}
\bP(\sup_{\theta\in\Theta}\Big|\frac{1}{I}\sum_{i=1}^{I} D_\theta\big(G_\eta(Z_{t_i,t_i+K},x_{t_i}),x_{t_i}\big)&-\bE D_\theta\big(G_\eta(Z_{t_i,t_i+K},x_{t_i}),x_{t_i}\big)\Big|\geq \varepsilon/2)\\
&\leq  2|\cN_\Theta|\exp\Big(-\frac{I\varepsilon^2}{576\tilde{L}^2(B^2_{f}+B^2_x)(K+\Delta_\beta)^2}\Big)
\end{align*}

Let us denote by $D_{\theta*}$  the optimal discriminator of $\mathbb{E}^{f,R} -  \mathbb{E}^{f,G_{\eta}}$. It is easy to see 
\begin{align*}
\mathscr{D}_\Theta\Big(\hat{\cP}_R(I),\hat{\cP}_{G_\eta,Z}(I)\Big)&\geq \Big|\frac{1}{I}\sum_{i=1}^{I} D_{\theta*}\big(G_\eta(Z_{t_i,t_i+K},x_{t_i}),x_{t_i}\big)-\frac{1}{I}\sum_{i=1}^{I} D_{\theta*}\big(f(R_{t_i+1:t_i+K},x_{t_i}),x_{t_i}\big)\Big|\\\
&\geq \sD_\Theta\Big(\cP_{R},\cP_{G_\eta,Z}\Big)-\sup_{\theta\in\Theta}\Big|\frac{1}{I}\sum_{i=1}^{I} D_\theta\big(f(R_{t_i+1:t_i+K},x_{t_i}),x_{t_i}\big)-\bE D_\theta\big(f(R_{t_i+1:t_i+K},x_{t_i}),x_{t_i}\big)\Big|\\
&-\sup_{\theta\in\Theta}\Big|\frac{1}{I}\sum_{i=1}^{I} D_\theta\big(G_\eta(Z_{t_i,t_i+K},x_{t_i}),x_{t_i}\big)-\bE D_\theta\big(G_\eta(Z_{t_i,t_i+K},x_{t_i}),x_{t_i}\big)\Big|\\
&\geq  \mathscr{D}_\Theta\Big(\cP_{R},\cP_{G_\eta,Z}\Big)-\varepsilon.
\end{align*}
The other direction is similar.
\end{proof}

Thus, if we further apply the uniform bound to the generator part, we can obtain the following theorem.
\begin{theorem}
Under the assumptions in subsection ``Assumptions and implications", let $G_{\eta_1},G_{\eta_2},\cdots,G_{\eta_{M}}$ denote the  generators in each of the $M$ iterations of the training procedure, and let $B_*=\sqrt{B^2_{f}+B^2_x}(K+\Delta_\beta)$, then
\begin{equation*}
\sup_{j\in[M]}\left|\sD_\Theta\Big(\hat{\cP}_R(I),\hat{\cP}_{G_\eta,Z}(I)\Big)-\sD_\Theta\Big(\cP_{R},\cP_{G_\eta,Z}\Big)\right|\leq \varepsilon,
\end{equation*}
with probability at least 
$$1-C\exp\Big(p\log(\frac{pL}{\varepsilon})\Big)(1+M)\exp\Big(-\frac{I\varepsilon^2}{\tilde{L}^2B^2_*}\Big),$$
for some universal constant $C>0$.
\end{theorem}

\newpage
\section{Additional Illustration for Experimental Results}

\subsection{Parameters for Financial Portfolio Choice and Simulated Time Series}
\label{appendix sim data generating process}
For non-synthetic data, $\hat{u}_{t+k|t}$ and $\hat{\Sigma}_{t+k|t}$ are estimators for the mean and co-variance of the  asset return at time $t+k$, defined as
\begin{align*}
   \!\! \hat{u}_{t+k|t}\!\! & = f_{u,k}(R_{t+1:t+k}, x_t) = \mbox{MA}_{\zeta}(r_{t+k})\\
\!\! \hat{\Sigma}_{t+k|t} \!\!& = f_{\Sigma,k}(R_{t+1:t+k}, x_t) \!=\! \mbox{MA}_{\zeta}(r_{t+k}r^\intercal_{t+k}) - \hat{u}_{t+k|t}^2,
\end{align*}
where $\mbox{MA}_{\zeta}(r_{t+k}) = \zeta\mbox{MA}_{\zeta}(r_{t+k-1}) + (1-\zeta) r_{t+k}$ with $\zeta = 0.74$.

For the synthetic data, the entire workflow is the same as with the non-synthetic data. The asset return $r'_{t+k|t}$ is generated by a GAN, where $z_{t,t+k}$ is the random seed. Similar to the non-synthetic data, we define 
\begin{align*}
   \!\! \hat{u}'_{t+k|t}\!\! & = f_{u,k}(R'_{t+1:t+k}, x_t) = \mbox{MA}_{\zeta}(r'_{t+k|t})\\
\!\! \hat{\Sigma}'_{t+k|t} \!\!& = f_{\Sigma,k}(R'_{t+1:t+k}, x_t) \!=\! \mbox{MA}_{\zeta}({r'}_{t+k|t}{r'}^\intercal_{t+k|t}) - \hat{u'}_{t+k|t}^2,
\end{align*}
where $\mbox{MA}_{\zeta}(r'_{t+1|t}) = \zeta\mbox{MA}_{\zeta}(r_t) + (1-\zeta) r'_{t+1|t}$.

The data-generating process for the simulated time series is given by  $r_{t+1} = b_0 r_t + \sum^{4}_{i=1} b_i \mbox{MA}_{\zeta_i}(r_t) + \epsilon$,
where $r_{t}$ is the asset return vector,  
$\mbox{MA}_{\zeta_i}(r_t)$ moving average operator, $\zeta_i$ smoothing parameter, $b_i$ coefficient, and $\epsilon$  noise. We set $\zeta_1 = 0.55, \zeta_2 = 0.74, \zeta_3 = 0.86, \zeta_4 = 0.92$, and $b_0 = 0.3, b_1 = 0.1, b_2 = 0.2, b_3 = 0.1$, and $b_4 = 0.1$.

We use a {\em multivariate t-distribution} to model the noise, with location parameter $\mu = [0, 0, 0, 0]^\intercal$,  shape matrix $\Sigma=$  $[1, 0.6, 0, 0; 0.6, 1, 0.6, 0; 0, 0.6, 1, 0.6; 0, 0, 0.6, 1]$,  and d.o.f.,~$\nu = 100$

\subsection{Conditioning Variables}
\label{conditioning variables appendix}
We use five features as the conditional variables for each asset: the  {\em asset return of the last week}, and four moving-average-based features, computed by taking the average of asset returns in the past few weeks. The moving average operators are defined as $\mbox{MA}_{\psi_i}(r_t)=\psi_i\mbox{MA}_{\psi_i}(r_{t-1}) + (1-\psi_i) r_t$, where $0 < \psi_i < 1, 1 \leq i \leq 4$ the smoothing parameter and $r_t$ are the raw asset returns. We use five handcrafted features to encode memory information over the time because our sample size is very limited for real data and we don't want the model capacity being too large thus over-fitting easily. 

\subsection{Training Details}
The total look ahead step is $K=4$. The learning rates are $\alpha = 1e-5$, $\omega_k = \lambda_{j,k} = 0.8^k$, where $1\leq k\leq K$. The number of inner loop iterations for the discriminators and the generator are $s_D = 1$ and $s_G = 5$. The clipping parameters are $l_b = -0.5, u_b = 0.5$. We use $N=2e5$ training steps ($2000$ training epochs) for the training process and the length of the time series is $T=3500$. The batch size is $I = 32$. 

We train on an Azure GPU standard NV6 instance, which has one Tesla M60 GPU.

\subsection{Generator Network Architecture}
\label{appendix generator network architecture}
For the generator, the neural network has dimension $(5+32)\times 16 \times 1$,  with  5 input nodes   based on the hand-crafted conditioning variables and 32 input nodes to provide random seeds.

The neural network architecture of the generator is shown in Figure \ref{fig_gen}. $x_{t,i}$, for $1\leq i \leq 5$, denotes the feature (state) variable input at time $t$ for the network. These are the  asset returns of the last day and four rolling-average based features, computed by taking the average of asset returns in the past few days mentioned in the ``Experimental setup" section and Section \ref{conditioning variables appendix}. Inputs $z_{t,t+k,i}$, for $1\leq k\leq K$, and $1\leq i \leq 32$, are random  seeds. The $h_{t+k|t,i}$ units are hidden ReLU nodes, and the output units, $r'_{t+k|t}$, provide the synthetic asset returns.

After obtaining $r'_{t+k|t}$, we  follow the recipe in the    ``DAT-CGAN for Financial Portfolio Choice" section in the main paper to compute quantities of interest, i.e., $\hat{u}'_{t+k|t} = f_{u,k}(R'_{t+1:t+k}, x_t)$, $\hat{\Sigma}'_{t+k|t} = f_{\Sigma,k}(R'_{t+1:t+k}, x_t)$, $\hat{H}'_{t+k|t}$, $w'_{t+k|t}$, $p'_{t+k+1|t}$ and $U'_{t+k+1|t}$.

\begin{figure}[h]
\centering
\begin{tikzpicture}[scale=1,transform shape, shorten >=1pt,->,draw=black!100, node distance=\layersep, thick] 
    \tikzstyle{input text}=[draw=white,minimum size=17pt,inner sep=0pt]
    \tikzstyle{input neuron}=[circle,draw=black!100,minimum size=25pt,inner sep=0pt,thick]
    \tikzstyle{hidden neuron1}=[draw=black!100,minimum size=25pt,inner sep=0pt,thick]
    \tikzstyle{hidden neuron}=[circle,draw=black!100,minimum size=25pt,inner sep=0pt,thick]
    \tikzstyle{hidden text}=[draw=white,minimum size=17pt,inner sep=0pt,thick]
    \tikzstyle{unit}=[circle,draw=black!100,minimum size=25pt,inner sep=0pt,thick]


    \node[input neuron, pin={[pin edge={<-}]left:$x_{t,1}$}] (I-1) at (0,1.5) {};
    \node[input text] (I-0) at (0,0.5) {$\vdots$};
    \node[input neuron, pin={[pin edge={<-}]left:$x_{t,5}$}] (I-2) at (0,-0.5) {};

    \node[input neuron, pin={[pin edge={<-}]left:$z_{t,t+1,1}$}] (I-3) at (0,-2.5) {};
    \node[input text] (I-0) at (0,-3.5) {$\vdots$};
    \node[input neuron, pin={[pin edge={<-}]left:$z_{t,t+1,32}$}] (I-4) at (0,-4.5) {};
    \node[input text] (I-0) at (0,-5.5) {$\vdots$};
    \node[input neuron, pin={[pin edge={<-}]left:$z_{t,t+K,1}$}] (I-5) at (0,-6.5) {};
    \node[input text] (I-0) at (0,-7.5) {$\vdots$};
    \node[input neuron, pin={[pin edge={<-}]left:$z_{t,t+K,32}$}] (I-6) at (0,-8.5) {};

    \path node[hidden neuron] (H-1) at (1.5,-2.5) {$h_{t+1|t,1}$};

    \node[input text] (H-0) at (1.5,-3.5) {$\vdots$};
    \path node[hidden neuron] (H-2) at (1.5,-4.5) {$h_{t+1|t,16}$};
    \node[input text] (H-0) at (1.5,-5.5) {$\vdots$};
    \path node[hidden neuron] (H-3) at (1.5,-6.5) {$h_{t+K|t,1}$};
    
    \node[input text] (H-0) at (1.5,-7.5) {$\vdots$};
    \path node[hidden neuron] (H-4) at (1.5,-8.5) {$h_{t+K|t,16}$};


    \path node[hidden neuron] (J-1) at (3,-2.5) {$r'_{t+1|t,1}$};
    \node[input text] (J-0) at (3,-3.5) {$\vdots$};
    \path node[hidden neuron] (J-2) at (3,-4.5) {$r'_{t+1|t,4}$};

    \node[input text] (J-0) at (3,-5.5) {$\vdots$};

    \path node[hidden neuron] (J-3) at (3,-6.5) {$r'_{t+K|t,1}$};
    \node[input text] (J-0) at (3,-7.5) {$\vdots$};
    \path node[hidden neuron] (J-4) at (3,-8.5) {$r'_{t+K|t,4}$};

    \foreach \src in {1,...,2}
        \foreach \dest in {1,...,4}
            \path (I-\src) edge (H-\dest);

    \foreach \src in {3,...,4}
        \foreach \dest in {1,...,2}
            \path (I-\src) edge (H-\dest);
            
    \foreach \src in {5,...,6}
        \foreach \dest in {3,...,4}
            \path (I-\src) edge (H-\dest);

    \foreach \src in {1,...,2}
        \foreach \dest in {1,...,2}
            \path (H-\src) edge (J-\dest);
            
    \foreach \src in {3,...,4}
        \foreach \dest in {3,...,4}
            \path (H-\src) edge (J-\dest);

    \draw[thick,dashed] ($(I-1.north west)+(-0.2,0.3)$)  rectangle ($(I-2.south east)+(0.2,-0.3)$) ;
    \draw[thick,dashed] ($(I-3.north west)+(-0.2,0.3)$)  rectangle ($(I-4.south east)+(0.2,-0.3)$) ;
    \draw[thick,dashed] ($(I-5.north west)+(-0.2,0.3)$)  rectangle ($(I-6.south east)+(0.2,-0.3)$) ;
    
    \draw[thick,dashed] ($(H-1.north west)+(-0.2,0.3)$)  rectangle ($(H-2.south east)+(0.2,-0.3)$) ;
    \draw[thick,dashed] ($(H-3.north west)+(-0.2,0.3)$)  rectangle ($(H-4.south east)+(0.2,-0.3)$) ;
    
    \draw[thick,dashed] ($(J-1.north west)+(-0.2,0.3)$)  rectangle ($(J-2.south east)+(0.2,-0.3)$) ;
    \draw[thick,dashed] ($(J-3.north west)+(-0.2,0.3)$)  rectangle ($(J-4.south east)+(0.2,-0.3)$) ;

\end{tikzpicture}
\caption{Generator Network Architecture. \label{fig_gen}}
\end{figure}

\subsection{Discriminator Network Architecture}
\label{appendix discriminator network architecture}
For the discriminator, the neural networks have  dimension $M\times  32 \times 1$, where $M$ matches the dimensions of the relevant quantity.

The neural network architecture of the discriminator for raw data is shown in Figure \ref{fig_dis}. It takes synthetic asset returns $r'_{t+k|t}$ or non-synthetic asset returns $r_{t+k}$ as one input, and $x_t$ as another input. The  $g_{t+k|t,i}$ units are hidden ReLU nodes, and the $D_{r,t+k|t}$ units are output nodes, representing the discriminator value for synthetic or non-synthetic asset returns.  We use the same discriminator architecture for all $k$.  

Similarly, the neural network architecture of the discriminator that takes synthetic utility $U'_{t+k|t}$ or non-synthetic utility $U_{t+k}$ as one input, and $x_t$ as another input, shown in Figure~\ref{fig_3}. The $m_{t+k|t,i}$ units are hidden nodes (ReLU), and the $D_{U,t+k|t}$ units are output nodes, representing the discriminator value for synthetic or non-synthetic  utility (decision related quantities). We use the same discriminator architecture for all $k$.  
\begin{figure*}[h!]
\centering
\begin{subfigure}[b]{0.49\textwidth}
\centering
\begin{tikzpicture}[scale=1,transform shape, shorten >=1pt,->,draw=black!100, node distance=\layersep, thick] 
    \tikzstyle{input text}=[draw=white,minimum size=17pt,inner sep=0pt]
    \tikzstyle{input neuron}=[circle,draw=black!100,minimum size=25pt,inner sep=0pt,thick]
    \tikzstyle{hidden neuron1}=[draw=black!100,minimum size=25pt,inner sep=0pt,thick]
    \tikzstyle{hidden neuron}=[circle,draw=black!100,minimum size=25pt,inner sep=0pt,thick]
    \tikzstyle{hidden text}=[draw=white,minimum size=17pt,inner sep=0pt,thick]
    \tikzstyle{unit}=[circle,draw=black!100,minimum size=25pt,inner sep=0pt,thick]

    \node[input neuron, pin={[pin edge={<-}]left:$x_{t,1}$}] (I-1) at (0,1.5) {};
    \node[input text] (I-0) at (0,0.5) {$\vdots$};
    \node[input neuron, pin={[pin edge={<-}]left:$x_{t,5}$}] (I-2) at (0,-0.5) {};

    \node[input neuron, pin={[pin edge={<-}]left:$r'_{t+1|t,1}$($r_{t+1,1}$)}] (I-3) at (0,-2.5) {};
    \node[input text] (I-0) at (0,-3.5) {$\vdots$};
    \node[input neuron, pin={[pin edge={<-}]left:$r'_{t+1|t,4}$($r_{t+1,4}$)}] (I-4) at (0,-4.5) {};
    \node[input text] (I-0) at (0,-5.4) {$\vdots$};
    \node[input neuron, pin={[pin edge={<-}]left:$r'_{t+K|t,1}$($r_{t+K,1}$)}] (I-5) at (0,-6.5) {};
    \node[input text] (I-0) at (0,-7.5) {$\vdots$};
    \node[input neuron, pin={[pin edge={<-}]left:$r'_{t+K|t,4}$($r_{t+K,4}$)}] (I-6) at (0,-8.5) {};

    \path node[hidden neuron] (H-1) at (1.5,-2.5) {$g_{t+1|t,1}$};

    \node[input text] (H-0) at (1.5,-3.5) {$\vdots$};
    \path node[hidden neuron] (H-2) at (1.5,-4.5) {$g_{t+1|t,32}$};
    \node[input text] (H-0) at (1.5,-5.4) {$\vdots$};
    \path node[hidden neuron] (H-3) at (1.5,-6.5) {$g_{t+K|t,1}$};
    
    \node[input text] (H-0) at (1.5,-7.5) {$\vdots$};
    \path node[hidden neuron] (H-4) at (1.5,-8.5) {$g_{t+K|t,32}$};


    \path node[hidden neuron] (J-1) at (3,-3.5) {$D_{r,t+1|t}$};

    \node[input text] (J-0) at (3,-5.4) {$\vdots$};

    \path node[hidden neuron] (J-2) at (3,-7.5) {$D_{r,t+K|t}$};

    \foreach \src in {1,...,2}
        \foreach \dest in {1,...,4}
            \path (I-\src) edge (H-\dest);

    \foreach \src in {3,...,4}
        \foreach \dest in {1,...,2}
            \path (I-\src) edge (H-\dest);
            
    \foreach \src in {5,...,6}
        \foreach \dest in {3,...,4}
            \path (I-\src) edge (H-\dest);

    \foreach \src in {1,...,2}
        \foreach \dest in {1}
            \path (H-\src) edge (J-\dest);
            
    \foreach \src in {3,...,4}
        \foreach \dest in {2}
            \path (H-\src) edge (J-\dest);

    \draw[thick,dashed] ($(I-1.north west)+(-0.2,0.3)$)  rectangle ($(I-2.south east)+(0.2,-0.3)$) ;
    \draw[thick,dashed] ($(I-3.north west)+(-0.2,0.3)$)  rectangle ($(I-4.south east)+(0.2,-0.3)$) ;
    \draw[thick,dashed] ($(I-5.north west)+(-0.2,0.3)$)  rectangle ($(I-6.south east)+(0.2,-0.3)$) ;
    
    \draw[thick,dashed] ($(H-1.north west)+(-0.2,0.3)$)  rectangle ($(H-2.south east)+(0.2,-0.3)$) ;
    \draw[thick,dashed] ($(H-3.north west)+(-0.2,0.3)$)  rectangle ($(H-4.south east)+(0.2,-0.3)$) ;

\end{tikzpicture}
\caption{Discriminator Network Architecture for asset returns. \label{fig_dis}}
\end{subfigure}
\begin{subfigure}[b]{0.49\textwidth}
\centering
\begin{tikzpicture}[scale=1,transform shape, shorten >=1pt,->,draw=black!100, node distance=\layersep, thick] 
    \tikzstyle{input text}=[draw=white,minimum size=17pt,inner sep=0pt]
    \tikzstyle{input neuron}=[circle,draw=black!100,minimum size=25pt,inner sep=0pt,thick]
    \tikzstyle{hidden neuron1}=[draw=black!100,minimum size=25pt,inner sep=0pt,thick]
    \tikzstyle{hidden neuron}=[circle,draw=black!100,minimum size=25pt,inner sep=0pt,thick]
    \tikzstyle{hidden text}=[draw=white,minimum size=17pt,inner sep=0pt,thick]
    \tikzstyle{unit}=[circle,draw=black!100,minimum size=25pt,inner sep=0pt,thick]


    \node[input neuron, pin={[pin edge={<-}]left:$x_{t,1}$}] (I-1) at (0,1.5) {};
    \node[input text] (I-0) at (0,0.5) {$\vdots$};
    \node[input neuron, pin={[pin edge={<-}]left:$x_{t,5}$}] (I-2) at (0,-0.5) {};

    \node[input neuron, pin={[pin edge={<-}]left:$U'_{t+1|t}$($U_{t+1}$)}] (I-3) at (0,-3.5) {};
    \node[input text] (I-0) at (0,-5.5) {$\vdots$};
    \node[input neuron, pin={[pin edge={<-}]left:$U'_{t+K|t}$($U_{t+K}$)}] (I-4) at (0,-7.5) {};
    \path node[hidden neuron] (H-1) at (1.5,-2.5) {$m_{t+1|t,1}$};

    \node[input text] (H-0) at (1.5,-3.5) {$\vdots$};
    \path node[hidden neuron] (H-2) at (1.5,-4.5) {$m_{t+1|t,32}$};
    \node[input text] (H-0) at (1.5,-5.5) {$\vdots$};
    \path node[hidden neuron] (H-3) at (1.5,-6.5) {$m_{t+K|t,1}$};
    
    \node[input text] (H-0) at (1.5,-7.5) {$\vdots$};
    \path node[hidden neuron] (H-4) at (1.5,-8.5) {$m_{t+K|t,32}$};


    \path node[hidden neuron] (J-1) at (3,-3.5) {$D_{U,t+1|t}$};
    \node[input text] (J-0) at (3,-5.5) {$\vdots$};
    \path node[hidden neuron] (J-2) at (3,-7.5) {$D_{U,t+K|t}$};

    \foreach \src in {1,...,2}
        \foreach \dest in {1,...,4}
            \path (I-\src) edge (H-\dest);

    \foreach \src in {3}
        \foreach \dest in {1,...,2}
            \path (I-\src) edge (H-\dest);
            
    \foreach \src in {4}
        \foreach \dest in {3,...,4}
            \path (I-\src) edge (H-\dest);

    \foreach \src in {1,...,2}
        \foreach \dest in {1}
            \path (H-\src) edge (J-\dest);
            
    \foreach \src in {3,...,4}
        \foreach \dest in {2}
            \path (H-\src) edge (J-\dest);

    \draw[thick,dashed] ($(I-1.north west)+(-0.2,0.3)$)  rectangle ($(I-2.south east)+(0.2,-0.3)$) ;

    \draw[thick,dashed] ($(H-1.north west)+(-0.2,0.3)$)  rectangle ($(H-2.south east)+(0.2,-0.3)$) ;
    \draw[thick,dashed] ($(H-3.north west)+(-0.2,0.3)$)  rectangle ($(H-4.south east)+(0.2,-0.3)$) ;

\end{tikzpicture}
\caption{Discriminator Network Architecture for utility. \label{fig_3}}
\end{subfigure}
\caption{Discriminator Network Architecture}
\end{figure*}

\begin{figure*}[h!]
\centering
\begin{subfigure}[b]{0.32\textwidth}
\centering
\includegraphics[width = \textwidth]{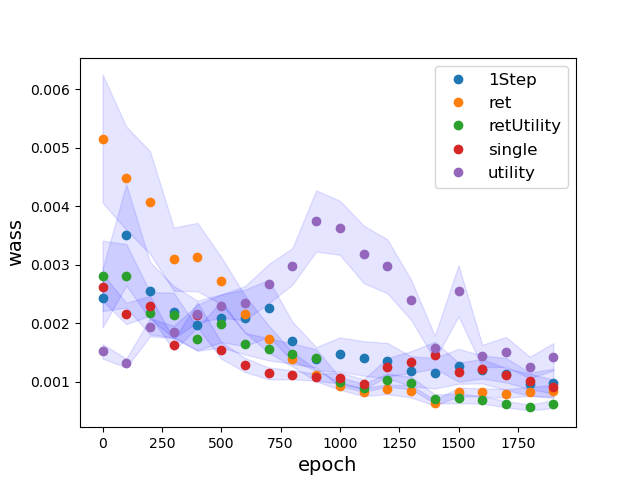}
\caption{Initial step asset returns~~~~~~~~~ \label{fig_sim2_ret1}}
\end{subfigure}
\hfill
\begin{subfigure}[b]{0.32\textwidth}
\centering
\includegraphics[width = 1\textwidth]{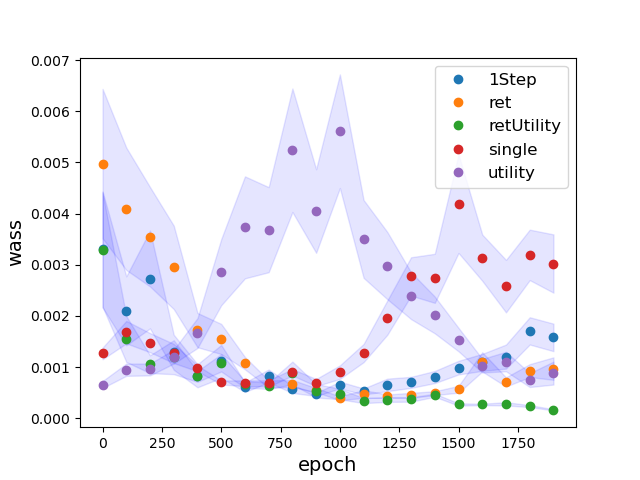}
\caption{Initial step precision matrix \label{fig_sim2_invCov1}}
\end{subfigure}
\hfill
\begin{subfigure}[b]{0.32\textwidth}
\centering
\includegraphics[width = 1\textwidth]{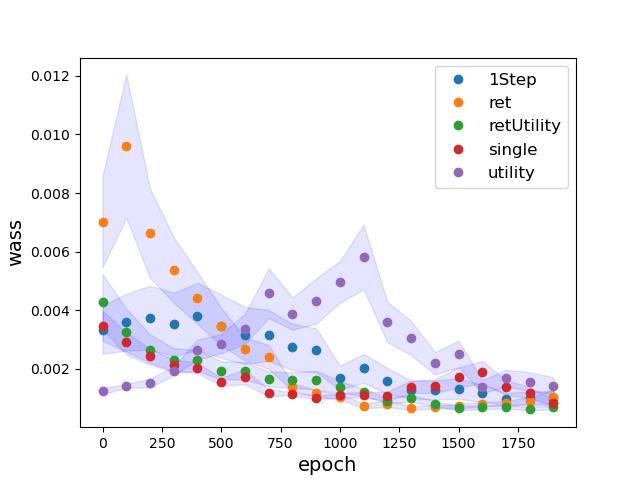}
\caption{Initial step portfolio weights \label{fig_sim2_portW1}}
\end{subfigure}
\vskip\baselineskip
\begin{subfigure}[b]{0.32\textwidth}
\centering
\includegraphics[width = \textwidth]{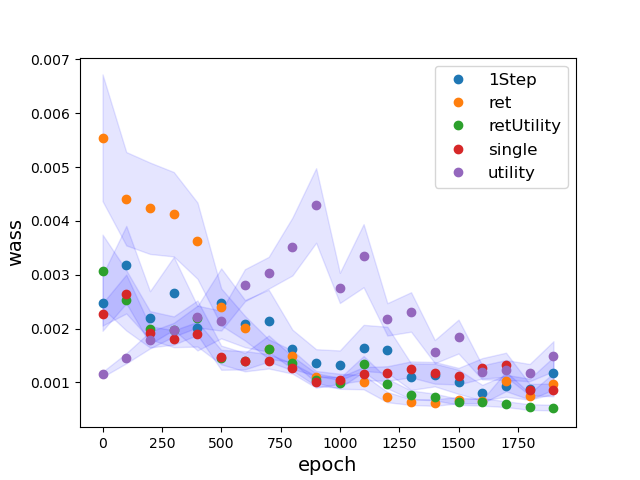}
\caption{Final step asset returns~~~~~~~~~ \label{fig_sim2_ret4}}
\end{subfigure}
\hfill
\begin{subfigure}[b]{0.32\textwidth}
\centering
\includegraphics[width = 1\textwidth]{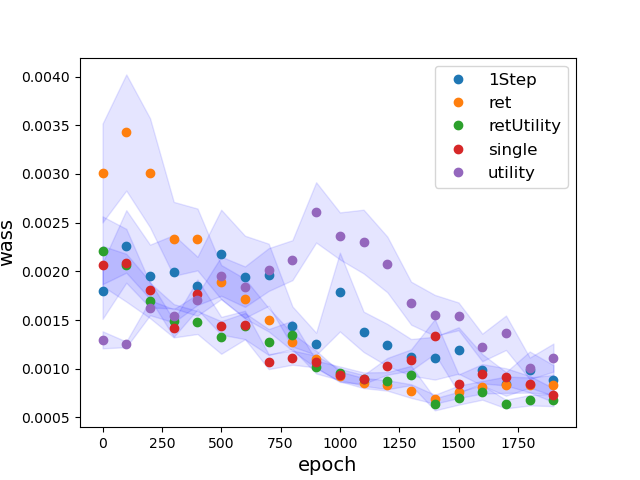}
\caption{Final step precision matrix~~~~ \label{fig_sim2_invCov3}}
\end{subfigure}
\hfill
\begin{subfigure}[b]{0.32\textwidth}
\centering
\includegraphics[width = 1\textwidth]{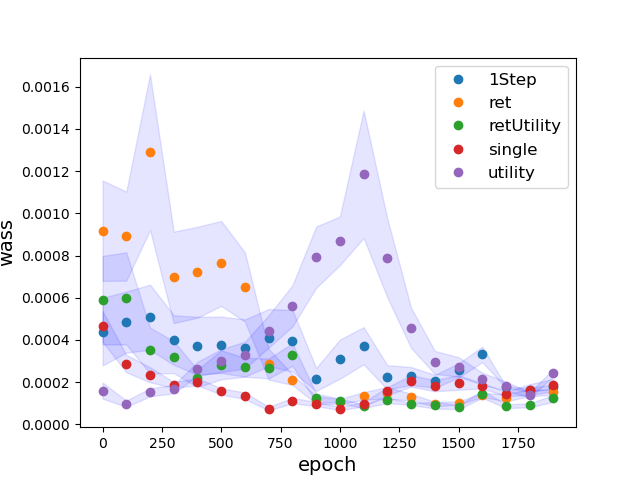}
\caption{Final step portfolio weights \label{fig_sim2_portW3}}
\end{subfigure}
\caption{\label{fig_sim2}
Wasserstein Distance between non-synthetic and synthetic data for simulated time series. \ref{fig_sim2_ret1} and \ref{fig_sim2_ret4}  asset returns; \ref{fig_sim2_invCov1} and \ref{fig_sim2_invCov3}  estimated precision matrix; \ref{fig_sim2_portW1} and \ref{fig_sim2_portW3}  estimated portfolio weights.  An epoch is one full pass of the data, and the shaded areas are confidence bands computed over 5 runs. }
\end{figure*}

\subsection{Additional Experiments}
\label{experiment appendix}
To study how would a different decision related quantity would affect our model performance, We use the a similar setting for simulated data experiment as in Section \ref{sim experiment} but set the risk preference parameter $\phi = 0.5$.

Figures~\ref{fig_sim2_ret1} to~\ref{fig_sim2_portW3} confirm that the Ret-Utility-GAN 
has the best performance in terms of minimizing 
Wasserstein distance for each of asset returns, estimated precision matrix, and portfolio weights.   It
 performs (1) better than Ret-GAN, confirming  that introducing the decision-related quantity, utility,   provides useful moderation on the distribution on  synthetic data, and also stabilize the training process; (2) better than the Utility-GAN, which shows that also including the asset return loss  helps; (3)
 better than the  Single-GAN, which shows that imposing loss for each quantity is more effective
than a single loss  on stacked quantities; and 
(4) better than the 1step-GAN, which 
shows that the Ret-Utility-GAN is effective in addressing   exposure bias.

This experiment shows that our DAT-CGANs are applicable to users with different preference, and it performs better than the standard baselines. 

\end{document}